\documentclass{article}

    \PassOptionsToPackage{numbers, compress}{natbib}

\usepackage[final]{neurips_2022}

\usepackage[utf8]{inputenc} 
\usepackage[T1]{fontenc}    
\usepackage{hyperref}       
\usepackage{url}            
\usepackage{booktabs}       
\usepackage{amsfonts}       
\usepackage{nicefrac}       
\usepackage{microtype}      
\usepackage{xcolor}         
\usepackage{subcaption}
\usepackage{natbib}
\usepackage{hyperref}
\hypersetup{
    colorlinks=true,
    citecolor=black,
    linkcolor=red,
    urlcolor=cyan
}

\title{Towards Understanding Grokking:\\
An Effective Theory of Representation Learning}

\author{%
  Ziming Liu, Ouail Kitouni, Niklas Nolte, Eric J. Michaud, Max Tegmark, Mike Williams \\ 
  Department of Physics, Institute for AI and Fundamental Interactions, MIT \\
  \texttt{\{zmliu,kitouni,nnolte,ericjm,tegmark,mwill\}@mit.edu} \\
}

\usepackage{graphicx}
\usepackage{dcolumn}
\usepackage{bm}
\usepackage{float}
\usepackage{subcaption}
\usepackage{multirow}
\usepackage{comment}
\usepackage{dsfont}
\usepackage{amsthm}
\usepackage{xcolor}
\usepackage{array}
\usepackage{eucal}
\usepackage{makecell}
\usepackage{mathtools}
\usepackage{makecell}
\usepackage{enumitem}
\setlist{leftmargin=10mm}

\newcommand{\mat}[1]{\mathbf{#1}}

\newtheorem{definition}{Definition}
\newtheorem{proposition}{Proposition}

\newcommand{\todo}[1]{\textcolor{blue}{#1}}

\def\spose#1{\hbox to 0pt{#1\hss}}
\def\simlt{\mathrel{\spose{\lower 3pt\hbox{$\mathchar"218$}}
     \raise 2.0pt\hbox{$\mathchar"13C$}}}
\def\simgt{\mathrel{\spose{\lower 3pt\hbox{$\mathchar"218$}}
     \raise 2.0pt\hbox{$\mathchar"13E$}}}
\def\simpropto{\mathrel{\spose{\lower 3pt\hbox{$\mathchar"218$}}
     \raise 2.0pt\hbox{$\propto$}}}

\def\beq#1{\begin{equation}\label{#1}}
\def\eeq{\end{equation}}
\def\beqa#1{\begin{eqnarray}\label{#1}}
\def\eeqa{\end{eqnarray}}

\begin{document}

\maketitle

\begin{abstract}
    We aim to understand \emph{grokking}, a phenomenon where models generalize long after overfitting their training set.
    We present both a \textit{microscopic} analysis anchored by an effective theory and a \textit{macroscopic} analysis of phase diagrams describing learning performance across hyperparameters. We find that generalization originates from structured representations whose training dynamics and dependence on training set size can be predicted by our effective theory in a toy setting. We observe empirically the presence of four learning phases: \textit{comprehension}, \textit{grokking}, \textit{memorization}, and \textit{confusion}. We find representation learning to occur only in a ``Goldilocks zone'' (including comprehension and grokking) between memorization and confusion. We find on transformers the grokking phase stays closer to the memorization phase (compared to the comprehension phase), leading to delayed generalization. The Goldilocks phase is reminiscent of ``intelligence from starvation'' in Darwinian evolution, where resource limitations drive discovery of more efficient solutions. This study not only provides intuitive explanations of the origin of grokking, but also highlights the usefulness of physics-inspired tools, e.g., effective theories and phase diagrams, for understanding deep learning.
\end{abstract}

\section{Introduction}

Perhaps \emph{the} central challenge of a scientific understanding of deep learning lies in accounting for neural network generalization. 
Power et al.~\cite{power2022grokking} recently added a new puzzle to the task of understanding generalization with their discovery of \emph{grokking}. Grokking refers to the surprising phenomenon of \textit{delayed generalization} where neural networks, on certain learning problems, generalize long after overfitting their training set. It is a rare albeit striking phenomenon that violates common machine learning intuitions, raising three key puzzles:

\begin{itemize}
  \item[\bf Q1]\textit{The origin of generalization}: When trained on the algorithmic datasets where grokking occurs, how do models generalize at all? 
  \item[\bf Q2]\textit{The critical training size}: Why does the training time needed to ``grok'' (generalize) diverge as the training set size decreases toward a critical point?
  \item[\bf Q3]\textit{Delayed generalization}: Under what conditions does delayed generalization occur? 
\end{itemize}

We provide evidence that representation learning is central to answering each of these questions. Our answers can be summarized as follows:

\begin{itemize}
  \item[\bf A1] Generalization can be attributed to learning a good representation of the input embeddings, i.e., a representation that has the appropriate structure for the task and which can be predicted from the theory in Section \ref{eq:RQI}. See Figures~\ref{fig:ring} and \ref{fig:toys}.
  \item[\bf A2] The critical training set size corresponds to the least amount of training data that can determine such a representation (which, in some cases, is unique up to linear transformations).
  \item[\bf A3] Grokking is a phase between ``comprehension'' and  ``memorization'' phases and it can be remedied with proper hyperparmeter tuning, as illustrated by the phase diagrams in Figure~\ref{fig:grokking_pd}.   
\end{itemize}

\begin{figure}[t]
    \centering
    \includegraphics[width=1\linewidth]{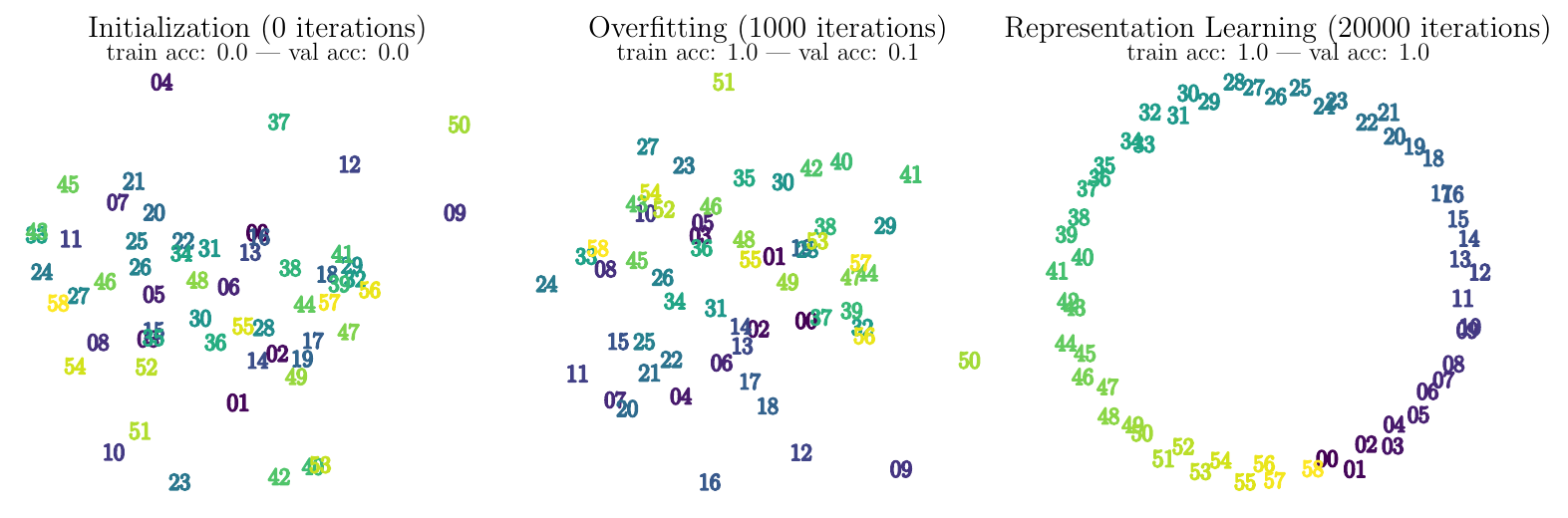}
    \caption{Visualization of the first two principal components of the learned input embeddings at different training stages of a transformer learning modular addition. We observe that generalization coincides with the emergence of structure in the embeddings. See Section \ref{sec:beyond_toy} for the training details.}
    \label{fig:ring}
\end{figure}

This paper is organized as follows: In Section \ref{sec:problem_setting}, we introduce the problem setting and build a simplified toy model. In Section \ref{sec:represenation}, we will use an \textit{effective theory} approach, a useful tool from theoretical physics, to shed some light on questions {\bf Q1} and {\bf Q2} and show the relationship between generalization and the learning of structured representations. 
In Section \ref{sec:phase_diagram}, we explain {\bf Q3} by displaying phase diagrams from a grid search of hyperparameters and show how we can ``de-delay'' generalization by following intuition developed from the phase diagram. We discuss related work in Section \ref{sec:related_works}, followed by conclusions in Section \ref{sec:conclusions}.\footnote{Project code can be found at: 
\url{https://github.com/ejmichaud/grokking-squared}
}

\section{Problem Setting}\label{sec:problem_setting}

Power et al.~\cite{power2022grokking} observe grokking on a less common task -- learning ``algorithmic'' binary operations.
Given some binary operation $\circ$, a network is tasked with learning the map $(a, b) \mapsto c$ where $c = a \circ b$.
They use a decoder-only transformer to predict the second to last token in a tokenized equation of the form ``<lhs> <op> <rhs> <eq> \textbf{<result>} <eos>''.
Each token is represented as a 256-dimensional embedding vector.
The embeddings are learnable and initialized randomly.
After the transformer, a final linear layer maps the output to class logits for each token.

{\bf Toy Model} We primarily study grokking in a simpler toy model, which still retains the key behaviors from the setup of~\cite{power2022grokking}.
Although~\cite{power2022grokking} treated this as a classification task, we study both regression (mean-squared error) and classification (cross-entropy). The basic setup is as follows:
our model takes as input the symbols $a, b$ and maps them to trainable embedding vectors $ \mathbf{E}_a, \mathbf{E}_b \in \mathbb{R}^{d_{\text{in}}}$.
It then sums $ \mathbf{E}_a, \mathbf{E}_b$ and sends the resulting vector through a ``decoder'' MLP. The target output vector, denoted $\mathbf{Y}_c \in \mathbb{R}^{d_{\text{out}}}$
is a fixed random vector (regression task) or a one-hot vector (classification task). Our model architecture can therefore be compactly described as $(a, b) \mapsto {\rm Dec}(\mathbf{E}_a + \mathbf{E}_b)$, where the embeddings $\mathbf{E}_*$ and the decoder are trainable. Despite its simplicity, this toy model can generalize to all abelian groups (discussed in Appendix \ref{app:applicability}).
In sections \ref{sec:represenation}-\ref{sec:toyphases}, we consider only the binary operation of addition. We consider modular addition in Section~\ref{sec:beyond_toy} to generalize some of our results to a transformer architecture and study general non-abelian operations in Appendix \ref{app:non-abelian}.

{\bf Dataset} In our toy setting, we are concerned with learning the addition operation. A data sample corresponding to $i+j$ is denoted as $(i,j)$ for simplicity. If $i, j \in \{0, \ldots, p-1\}$, there are in total $p(p+1)/2$ different samples since we consider $i+j$ and $j+i$ to be the same sample. A dataset $D$ is a set of non-repeating data samples. We denote the full dataset as $D_0$ and split it into a training dataset $D$ and a validation dataset $D'$, i.e., $D\bigcup D'=D_0$, $D\bigcap D'=\emptyset$. We define \textit{training data fraction} $=|D|/|D_0|$ where $|\cdot|$ denotes the cardinality of the set.

\section{Why Generalization Occurs: Representations and Dynamics}\label{sec:represenation}

We can see that generalization appears to be linked to the emergence of highly-structured embeddings in Figure~\ref{fig:toys}. In particular, Figure~\ref{fig:toys} (a, b) shows parallelograms in toy addition, and (c, d) shows a circle in toy modular addition. We now restrict ourselves to the toy addition setup and formalize a notion of representation quality and show that it predicts the model's performance. We then develop a physics-inspired \emph{effective} theory of learning which can accurately predict the critical training set size and training trajectories of representations. The concept of an effective theory in physics is similar to model reduction in computational methods in that it aims to describe complex phenomena 
with simple yet intuitive pictures.
In our effective theory, we will model the dynamics of representation learning not as gradient descent of the true task loss but rather a simpler effective loss function $\ell_{\text{eff}}$ which depends only on the representations in embedding space and not on the decoder. 




\begin{figure}[t]
    \begin{subfigure}{.5\textwidth}
      \centering
      \includegraphics[width=.8\linewidth]{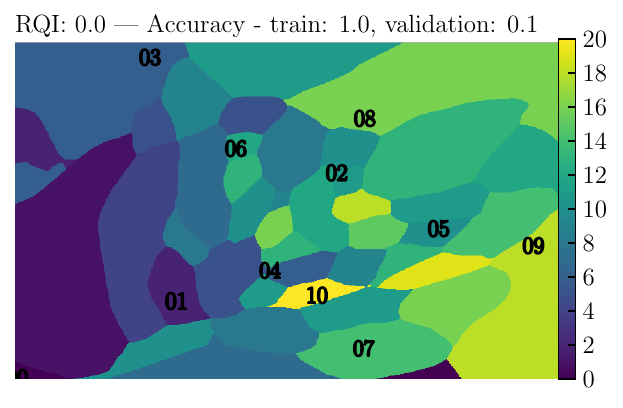}
      \vspace{-10pt}%
      \caption{Memorization in toy addition}
      \label{fig:Add2DBad}
    \end{subfigure}
    \begin{subfigure}{.5\textwidth}
      \centering
      \includegraphics[width=.8\linewidth]{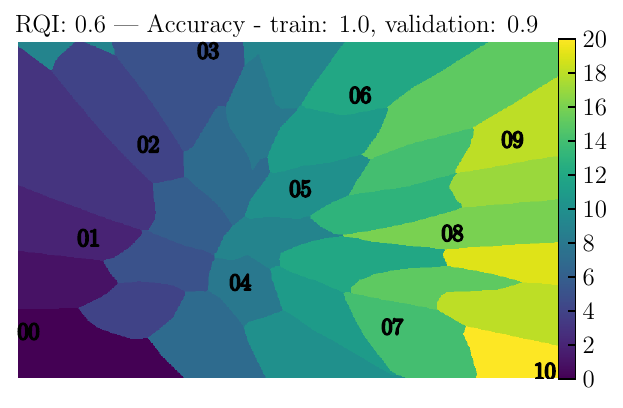}
      \vspace{-10pt}%
      \caption{Generalization in toy addition}
      \label{fig:Add2DGood}
    \end{subfigure}%
    
    \begin{subfigure}{.5\textwidth}
      \centering
      \includegraphics[width=.8\linewidth]{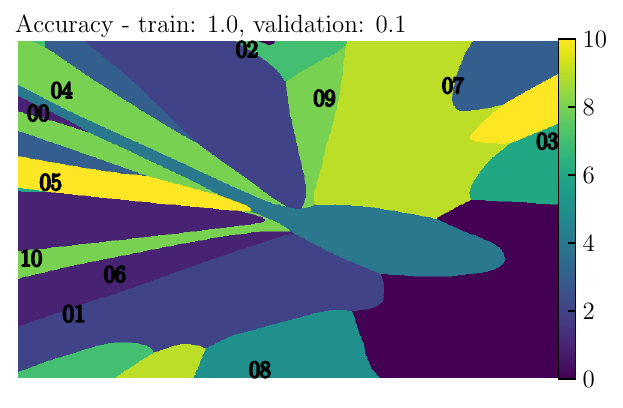}
      \vspace{-10pt}%
      \caption{Memorization in toy modular addition}
      \label{fig:ModAdd2DBad}
    \end{subfigure}
    \begin{subfigure}{.5\textwidth}
      \centering
      \includegraphics[width=.8\linewidth]{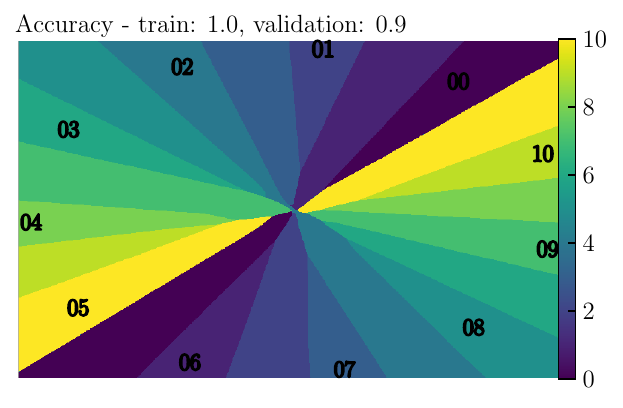}
      \vspace{-10pt}%
      \caption{Generalization in toy modular addition}
      \label{fig:ModAdd2DGood}
    \end{subfigure}%
    
    \caption{Visualization of the learned set of embeddings ($p=11$) and the decoder function associated with it for the case of 2D embeddings. 
    Axes refer to each dimension of the learned embeddings. The decoder is evaluated on a grid of points in embedding-space and the color at each point represents the highest probability class. For visualization purposes, the decoder is trained on inputs of the form $(\mathbf{E}_i + \mathbf{E}_j)/2$. One can read off the output of the decoder when fed the operation $i \circ j$ from this figure simply by taking the midpoint between the respective embeddings of $i$ and $j$. 
    }
    \label{fig:toys}
\end{figure}

\subsection{Representation quality predicts generalization for the toy model}
\label{sec:RQI}
A rigorous definition for \textit{structure} in the learned representation is necessary. We propose the following definition,
\begin{definition}
 $(i,j,m,n)$ is a $\delta$-\textbf{parallelogram} in the representation $\mathbf{R}\equiv[\mathbf{E}_0,\cdots,\mathbf{E}_{p-1}]$ if $$|(\mathbf{E}_i+\mathbf{E}_j)-(\mathbf{E}_m+\mathbf{E}_n)|\leq\delta.$$
\end{definition}
In the following derivations, we can take $\delta$, which is a small threshold to tolerate numerical errors, to be zero.

\begin{proposition}
    When the training loss is zero, any parallelogram $(i,j,m,n)$ in representation $\mathbf{R}$ satisfies $i+j=m+n$. 
\end{proposition}

\begin{proof}
    Suppose that this is not the case, i.e., suppose $\mathbf{E}_i + \mathbf{E}_j = \mathbf{E}_m + \mathbf{E}_n$ but $i +j \neq m + n$, then $\mathbf{Y}_{i+j} = \mathrm{Dec}(\mathbf{E}_i + \mathbf{E}_j) = \mathrm{Dec}(\mathbf{E}_m + \mathbf{E}_n) = \mathbf{Y}_{m+n}$ where the first and last equalities come from the zero training loss assumption. However, since $i + j \neq m + n$, we have $\mathbf{Y}_{i+j} \neq \mathbf{Y}_{n+m}$ (almost surely in the regression task), a contradiction.
\end{proof}

It is convenient to define the permissible parallelogram set associated with a training dataset $D$ (``permissible'' means consistent with 100\% training accuracy) as
\begin{equation}\label{eq:P0D}
    P_0(D) = \{(i,j,m,n)|(i,j)\in D,\,  (m,n)\in D,\,  i+j = m+n\}.
\end{equation}
For simplicity, we denote $P_0\equiv P_0(D_0)$. Given a representation $\mathbf{R}$, we can check how many permissible parallelograms actually exist in $\mathbf{R}$ within error $\delta$, so we define the parallelogram set corresponding to $\mathbf{R}$ as
\begin{equation}\label{eq:PR}
    P(\mathbf{R},\delta) = \{(i,j,m,n)|(i,j,m,n)\in P_0, |(\mathbf{E}_i+\mathbf{E}_j)-(\mathbf{E}_m+\mathbf{E}_n)|\leq\delta\}.
\end{equation}
For brevity we will write $P(\mathbf{R})$, suppressing the dependence on $\delta$. We define the representation quality index (RQI) as
\begin{equation}\label{eq:RQI}
    {\rm RQI}(\mathbf{R}) = \frac{|P(\mathbf{R})|}{|P_0|} \in [0, 1].
\end{equation}
We will use the term \textit{linear representation} or \textit{linear structure} to refer to a representation whose embeddings are of the form $\mathbf{E}_k=\mathbf{a}+k\mathbf{b} \ (k=0,\cdots,p-1;\mathbf{a},\mathbf{b}\in\mathbb{R}^{d_\text{in}})$. A linear representation has ${\rm RQI}=1$, while a random representation (sampled from, say, a normal dstribution) has ${\rm RQI}=0$ with high probability.

Quantitatively, we denote the ``predicted accuracy'' $\widehat{{\rm Acc}}$ as the accuracy achievable on the whole dataset given the representation $\mathbf{R}$ (see Appendix~\ref{app:acc} for the full details). In Figure~\ref{fig:Acc_DP}, we see that the predicted $\widehat{{\rm Acc}}$ aligns well with the true accuracy ${\rm Acc}$, establishing good evidence that structured representation of input embeddings leads to generalization. We use an example to illustrate the origin of generalization here. In the setup of Figure~\ref{fig:toys}~(b), suppose the decoder can achieve zero training loss and $\mat{E}_6+\mat{E}_8$ is a training sample hence ${\rm Dec}(\mat{E}_6+\mat{E}_8)=\mat{Y}_{14}$. At validation time, the decoder is tasked with predicting a validation sample $\mat{E}_5+\mat{E}_9$. Since $(5,9,6,8)$ forms a parallelogram such that $\mat{E}_5+\mat{E}_9=\mat{E}_6+\mat{E}_8$, the decoder can predict the validation sample correctly because ${\rm Dec}(\mat{E}_5+\mat{E}_9)={\rm Dec}(\mat{E}_6+\mat{E}_8)=\mat{Y}_{14}$.

\begin{figure}
    \centering
    \includegraphics[width=\linewidth]{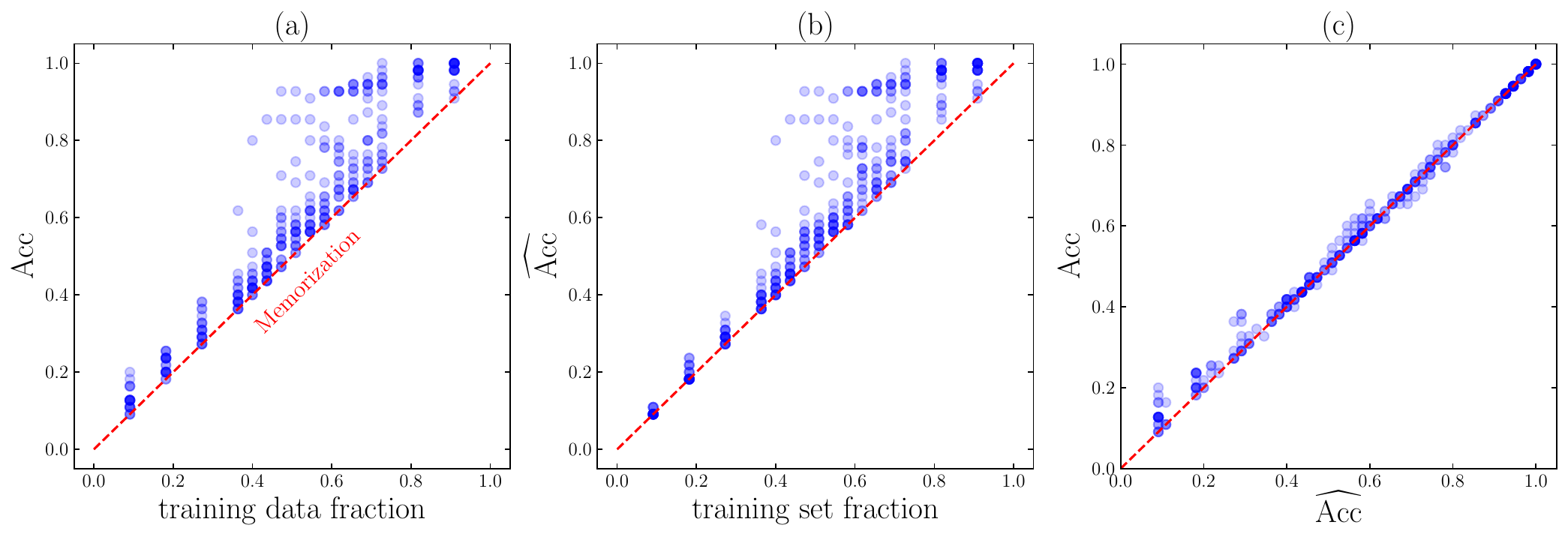}
    \caption{We compute accuracy (of the full dataset) either measured empirically ${\rm Acc}$, or predicted from the representation of the embeddings $\widehat{{\rm Acc}}$. These two accuracies as a function of training data fraction are plotted in (a)(b), and their agreement is shown in (c).}
    \label{fig:Acc_DP}
\end{figure}

\subsection{The dynamics of embedding vectors}\label{sec:effective_theory}

Suppose that we have an ideal model $\mathcal{M}^* = ({\rm Dec}^*, \mathbf{R}^*)$ such that:\footnote{One can verify a posteriori if a trained model $\mathcal{M}$ is close to being an ideal model $\mathcal{M}^*$. Please refer to Appendix \ref{app:ideal-gap} for details.}

\begin{itemize}
    \item (1) $\mathcal{M}^*$ can achieve zero training loss;
    \item (2) $\mathcal{M}^*$ has an injective decoder, i.e., ${\rm Dec^*}(\mathbf{x}_1)\neq {\rm Dec^*}(\mathbf{x}_2)$ for any $\mathbf{x}_1\neq \mathbf{x}_2$. 
\end{itemize}

Then Proposition~\ref{prop:train_parallel} provides a mechanism for the formation of parallelograms.

\begin{proposition}
\label{prop:train_parallel}
If a training set $D$ contains two samples $(i,j)$ and $(m,n)$ with $i+j=m+n$, then $\mathcal{M}^*$ learns a representation $\mathbf{R}^*$ such that $\mathbf{E}_i+\mathbf{E}_j=\mathbf{E}_m+\mathbf{E}_n$, i.e., $(i,j,m,n)$ forms a parallelogram.
\end{proposition}
\begin{proof}
Due to the zero training loss assumption, we have ${\rm Dec}^*(\mathbf{E}_i+\mathbf{E}_j)=\mathbf{Y}_{i+j}=\mathbf{Y}_{m+n}={\rm Dec}^*(\mathbf{E}_m+\mathbf{E}_n)$. Then the injectivity of ${\rm Dec^*}$ implies $\mathbf{E}_i+\mathbf{E}_j=\mathbf{E}_m+\mathbf{E}_n$.
\end{proof}


The dynamics of the trained embedding vectors are determined by various factors interacting in complex ways, for instance: the details of the decoder architecture, the optimizer hyperparameters, and the various kinds of implicit regularization induced by the training procedure. We will see that the dynamics of normalized quantities, namely, the normalized embeddings at time $t$, defined as $\tilde{\mathbf{E}}_k^{(t)} = \frac{\mathbf{E}_k^{(t)} - \mu_t}{\sigma_t}$, where $\mu_t = \frac{1}{p}\sum_k \mathbf{E}_k^{(t)}$ and $\sigma_t = \frac{1}{p}\sum_k |\mathbf{E}_k^{(t)} - \mu_t|^2$, can be qualitatively described by a simple effective loss (in the physics effective theory sense).
We will assume that the normalized embedding vectors obey a gradient flow for an effective loss function of the form
\begin{equation}\label{eq:Ei_dynamics}
    \frac{d\mathbf{\tilde E}_i}{dt} = -\frac{\partial \ell_{\text{eff}}}{\partial \mathbf{\tilde E}_i},
\end{equation}
\begin{equation}\label{eq:l_eff}
    \ell_{\text{eff}} = \frac{\ell_0}{Z_0},\quad  \ell_0\equiv\sum_{(i,j,m,n)\in P_0(D)}|\mathbf{\tilde E}_i+\mathbf{\tilde E}_j-\mathbf{\tilde E}_m-\mathbf{\tilde E}_n|^2/|P_0(D)|,\quad  Z_0\equiv \sum_{k}|\mathbf{\tilde E}_k|^2,
\end{equation}
where
$|\cdot|$ denotes Euclidean vector norm. Note that the embeddings do not collapse to the trivial solution $\mathbf{E}_0=\cdots=\mathbf{E}_{p-1}=0$ unless initialized as such, because two conserved quantities exist, as proven in Appendix~\ref{app:conservation-law}:
\begin{equation}
    \mathbf{C}=\sum_k \mathbf{E}_k, \quad 
    Z_0 = \sum_{k} |\mathbf{E}_k|^2.
\end{equation}

We shall now use the effective dynamics to explain empirical observations such as the existence of a critical training set size for generalization.


\begin{figure}[t]
    \centering
    \includegraphics[width=1\linewidth]{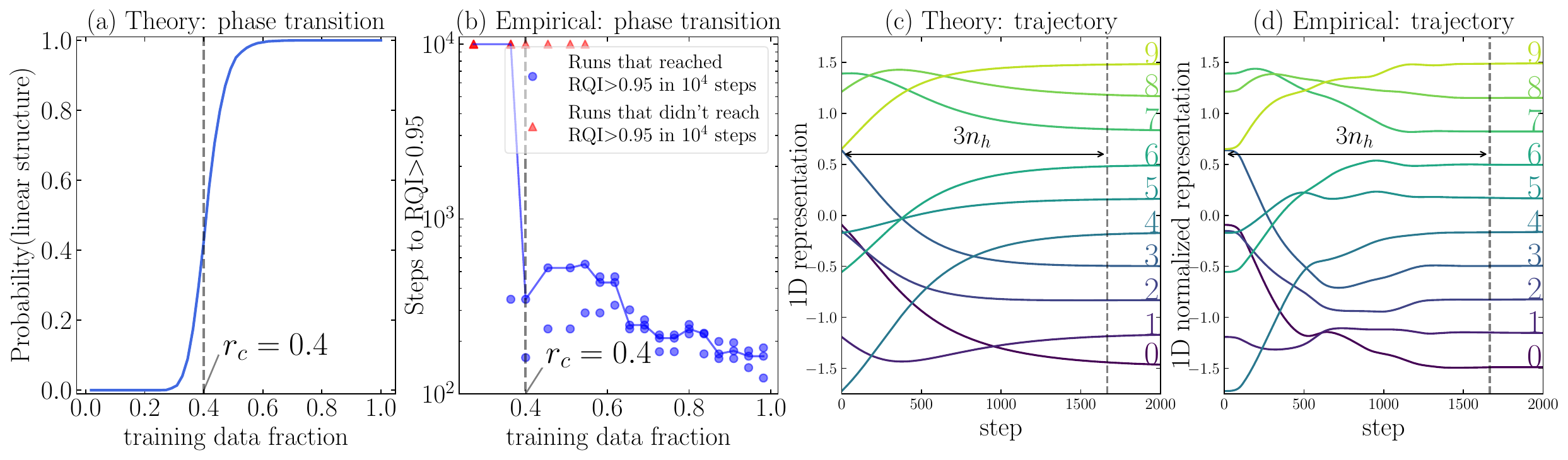}
    \caption{(a) The effective theory predicts a phase transition in the probability of obtaining a linear representation  around $r_c=0.4$. (b) Empirical results display a phase transition of RQI around $r_c=0.4$, in agreement with the theory (the blue line shows the median of multiple random seeds). The evolution of 1D representations predicted by the effective theory or obtained from neural network training (shown in (c) and (d) respectively) agree creditably well.}
    \label{fig:time}
\end{figure}

{\bf Degeneracy of ground states (loss optima)}  We define ground states as those representations satisfying $\ell_{\text{eff}}=0$, which requires the following linear equations to hold:
\begin{equation}\label{eq:A}
    A(P) = \{\mathbf{E}_i+\mathbf{E}_j=\mathbf{E}_m+\mathbf{E}_n|(i,j,m,n)\in P\}.
\end{equation}
Since each embedding dimension obeys the same set of linear equations, we will assume, without loss of generality, that $d_{\rm in}=1$. The dimension of the null space  of $A(P)$, denoted as $n_0$, is the number of degrees of freedom of the ground states. Given a set of parallelograms implied by a training dataset $D$, the nullity of $A(P(D))$ could be obtained by computing the singular values $0\leq \sigma_1\leq\cdots\leq\sigma_p$.
We always have $n_0\geq 2$, i.e., $\sigma_1=\sigma_{2}=0$ because the nullity of $A(P_0)$, the set of linear equations given by all possible parallelograms, is $\mathrm{Nullity} (A(P_0)) = 2$ which can be attributed to two degrees of freedom (translation and scaling).
If $n_0=2$, the representation is unique up to translations and scaling factors, and the embeddings have the form $\mathbf{E}_k=\mathbf{a}+k\mathbf{b}$. Otherwise, when  $n_0>2$, the representation is not constrained enough such that all the embeddings lie on a line.

We present theoretical predictions alongside empirical results for addition ($p=10$) in Figure~\ref{fig:time}. As shown in Figure~\ref{fig:time}~(a), our effective theory predicts that the probability that the training set implies a unique linear structure (which would result in perfect generalization) depends on the training data fraction and has a phase transition around $r_c=0.4$. Empirical results from training different models are shown in Figure~\ref{fig:time} (b). The number of steps to reach ${\rm RQI}>0.95$ is seen to have a phase transition at $r_c=0.4$, agreeing with the proposed effective theory and with the empirical findings in ~\cite{power2022grokking}. 

{\bf Time towards the linear structure} We define the Hessian matrix of $\ell_0$ as
\begin{equation}
    \mat{H}_{ij}=\frac{1}{Z_0}\frac{\partial^2 \ell_0}{\partial \mathbf{E}_i\partial \mathbf{E}_j},
\end{equation}
Note that $\ell_{\rm eff}=\frac{1}{2}\mat{R}^T\mat{H}\mat{R}$, $\mat{R}=[\mathbf{E}_0, \mathbf{E}_1,\cdots, \mathbf{E}_{p-1}]$, so the gradient descent is linear, i.e.,
\begin{equation}
    \frac{d\mat{R}}{dt} = - \mat{H}\mat{R}.
\end{equation}

If $\mat{H}$ has eigenvalues $\lambda_i=\sigma_i^2$ (sorted in increasing order) and eigenvectors $\bar{\mat{v}}_i$, and we have the initial condition $\mat{R}(t=0)=\sum_i a_i\bar{\mat{v}}_i$, then we have $\mat{R}(t)=\sum_i a_i\bar{\mat{v}}_ie^{-\lambda_it}$. The first two eigenvalues vanish and $t_h=1/\lambda_{3}$ determines the timescale for the slowest component to decrease by a factor of $e$. We call $\lambda_3$ the \textit{grokking rate}.   When the step size is $\eta$, the corresponding number of steps is $n_h=t_h/\eta=1/(\lambda_3\eta)$. 

We verify the above analysis with empirical results. Figure~\ref{fig:time} (c)(d) show the trajectories obtained from the effective theory and from neural network training, respectively. The 1D neural representation in Figure~\ref{fig:time} (d) are manually normalized to zero mean and unit variance. The two trajectories agree qualitatively, and it takes about $3n_h$ steps for two trajectories to converge to the linear structure. The quantitative differences might be due to the absence of the decoder in the effective theory, which assumes the decoder to take infinitesimal step sizes. 

{\bf Dependence of grokking on data size} Note that $\ell_{\rm eff}$ involves averaging over parallelograms in the training set, it is dependent on training data size, so is $\lambda_3$. In Figure \ref{fig:lambda3} (a), we plot the dependence of $\lambda_3$ on training data fraction. There are many datasets with the same data size, so $\lambda_3$ is a probabilistic function of data size. 

Two insights on grokking can be extracted from this plot: (i) When the data fraction is below some threshold (around 0.4), $\lambda_3$ is zero with high probability, corresponding to no generalization. This again verifies our critical point in Figure \ref{fig:time}. (ii) When data size is above the threshold, $\lambda_3$ (on average) is an increasing function of data size. This implies that grokking time $t\sim 1/\lambda_3$ decreases as training data size becomes larger, an important observation from ~\cite{power2022grokking}.

To verify our effective theory, we compare the grokking steps obtained from real neural network training (defined as steps to ${\rm RQI}>0.95$), and those predicted by our theory $t_{\rm th}\sim \frac{1}{\lambda_3\eta}$ ($\eta$ is the embedding learning rate), shown in Figure \ref{fig:lambda3} (b). The theory agrees qualitatively with neural networks, showing the trend of decreasing grokking steps as increasing data size. The quantitative differences might be explained as the gap between our effective loss and actual loss.

{\bf Limitations of the effective theory}
While our theory defines an effective loss based on the Euclidean distance between embeddings $\mat{E}_i + \mat{E}_j$ and $\mat{E}_n + \mat{E}_m$, one could imagine generalizing the theory to define a broader notion of parallogram given by some other metric on the representation space. For instance, if we have a decoder like in Figure~\ref{fig:toys} (d) then the distance between distinct representations within the same ``pizza slice'' is low, meaning that representations arranged not in parallelograms w.r.t. the Euclidean metric may be parallelograms with respect to the metric defined by the decoder.

\begin{figure}[t]
    \centering
    \begin{subfigure}[]{0.45\textwidth}
    \includegraphics[width=\linewidth]{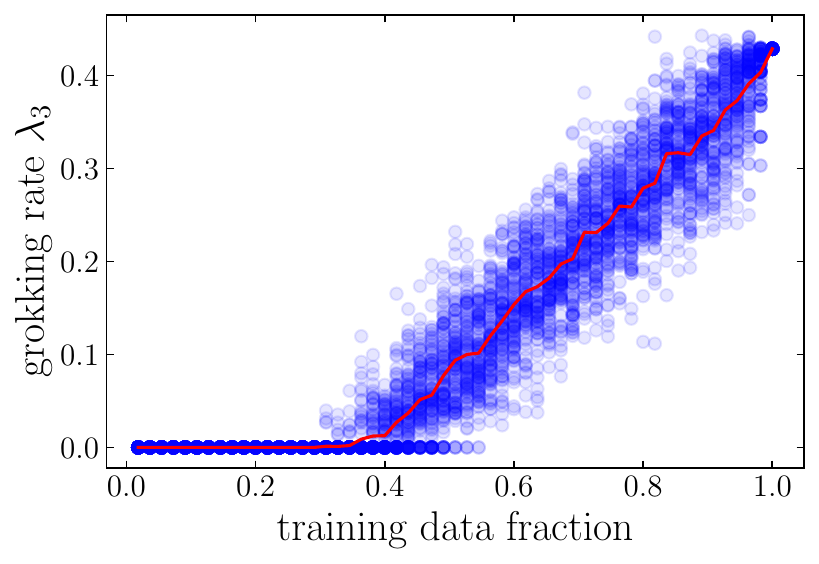}
    \caption{}
    \end{subfigure}
    \begin{subfigure}[]{0.45\textwidth}
    \includegraphics[width=\linewidth]{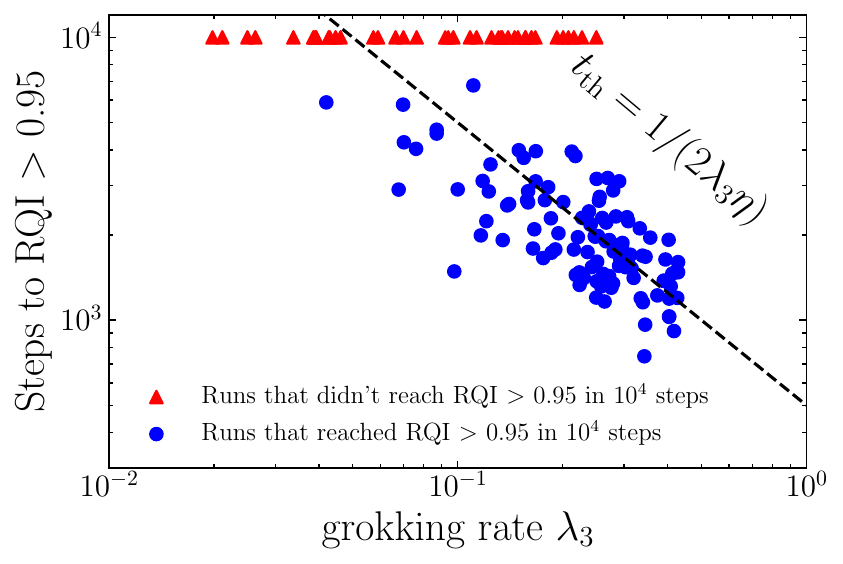}
    \caption{}
    \end{subfigure}
    \caption{Effective theory explains the dependence of grokking time on data size, for the addition task. (a) Dependence of $\lambda_3$ on training data fraction. Above the critical data fraction (around 0.4), as data size becomes larger, $\lambda_3$ increases hence grokking time $t\sim 1/\lambda_3$ (predicted by our effective theory) decreases. (b) Comparing grokking steps (defined as ${\rm RQI}>0.95$) predicted by the effective theory with real neural network results. $\eta=10^{-3}$ is the learning rate of the embeddings.}
    \label{fig:lambda3}
\end{figure}

\section{Delayed Generalization: A Phase Diagram}\label{sec:phase_diagram}

So far, we have (1) observed empirically that generalization on algorithmic datasets corresponds with the emergence of well-structured representations, (2) defined a notion of representation quality in a toy setting and shown that it predicts generalization, and (3) developed an effective theory to describe the learning dynamics of the representations in the same toy setting. We now study how optimizer hyperparameters affect high-level learning performance. In particular, we develop phase diagrams for how learning performance depends on the representation learning rate, decoder learning rate and the decoder weight decay. These parameters are of interest since they most explicitly regulate a kind of \emph{competition} between the encoder and decoder, as we elaborate below. 


\subsection{Phase diagram of a toy model}
\label{sec:toyphases}

{\bf Training details} We update the representation and the decoder with different optimizers. For the 1D embeddings, we use the Adam optimizer with learning rate $[10^{-5},10^{-2}]$ and zero weight decay. For the decoder, we use an AdamW optimizer with the learning rate in $[10^{-5},10^{-2}]$ and the weight decay in $[0,10]$ (regression) or $[0,20]$ (classification). For training/validation spliting, we choose 45/10 for non-modular addition ($p=10$) and 24/12 for the permutation group $S_3$. We hard-code addition or matrix multiplication (details in Appendix \ref{app:non-abelian}) in the decoder for the addition group and the permutation group, respectively. 

For each choice of learning rate and weight decay, we compute the number of steps to reach high (90\%) training/validation accuracy. The 2D plane is split into four phases: \textit{comprehension}, \textit{grokking}, \textit{memorization} and \textit{confusion}, defined in Table~\ref{tab:four_phases} in Appendix~\ref{app:defs_table}. Both comprehension and grokking are able to generalize (in the ``Goldilocks zone''), although the grokking phase has delayed generalization. Memorization is also called overfitting, and confusion means failure to even memorize training data. Figure~\ref{fig:grokking_pd} shows the phase diagrams for the addition group and the permutation group.
They display quite rich phenomena.

\begin{figure}[t]
    \centering
    \begin{subfigure}[]{0.35\textwidth}
    \includegraphics[width=\linewidth]{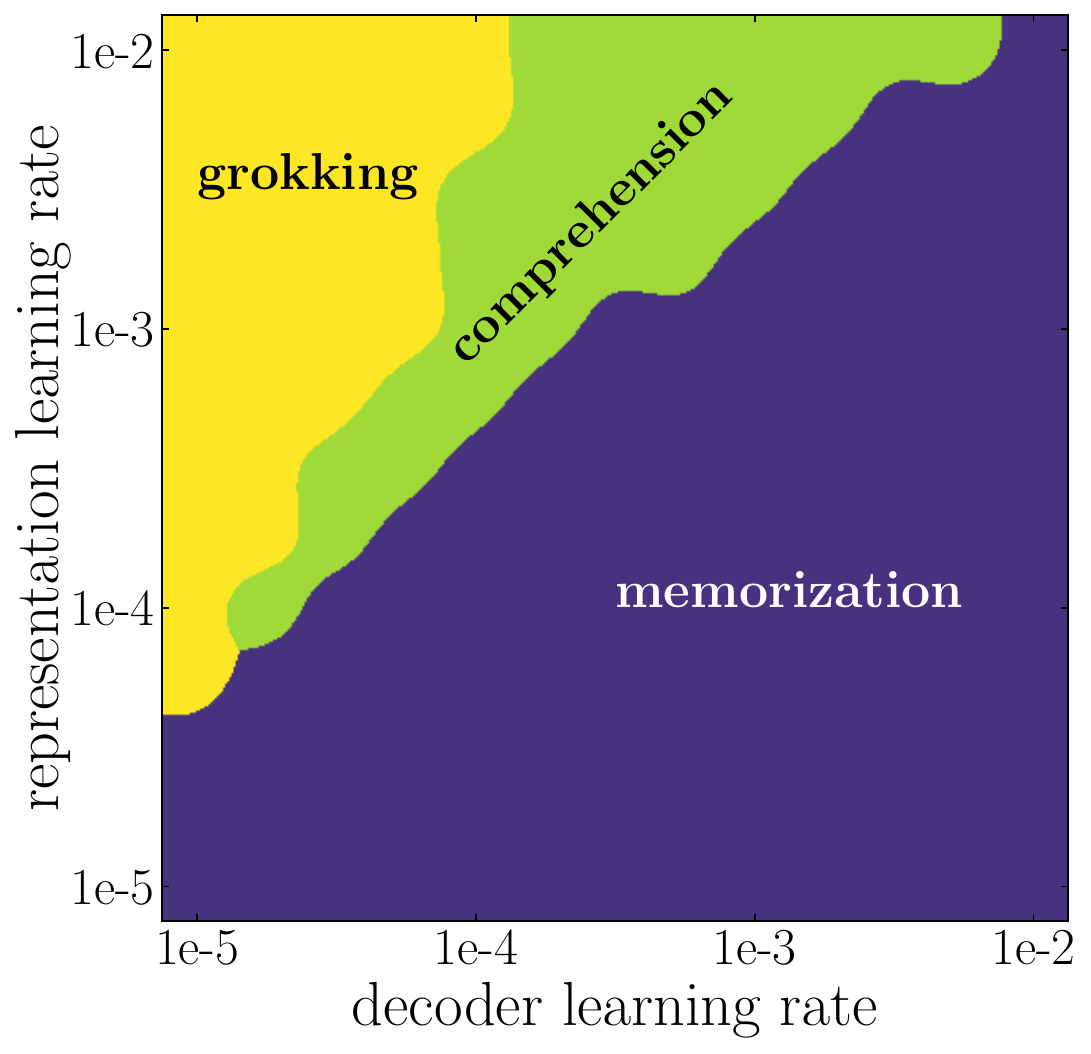}
    \caption{Addition, regression}
    \end{subfigure}
    \begin{subfigure}[]{0.35\textwidth}
    \includegraphics[width=\linewidth]{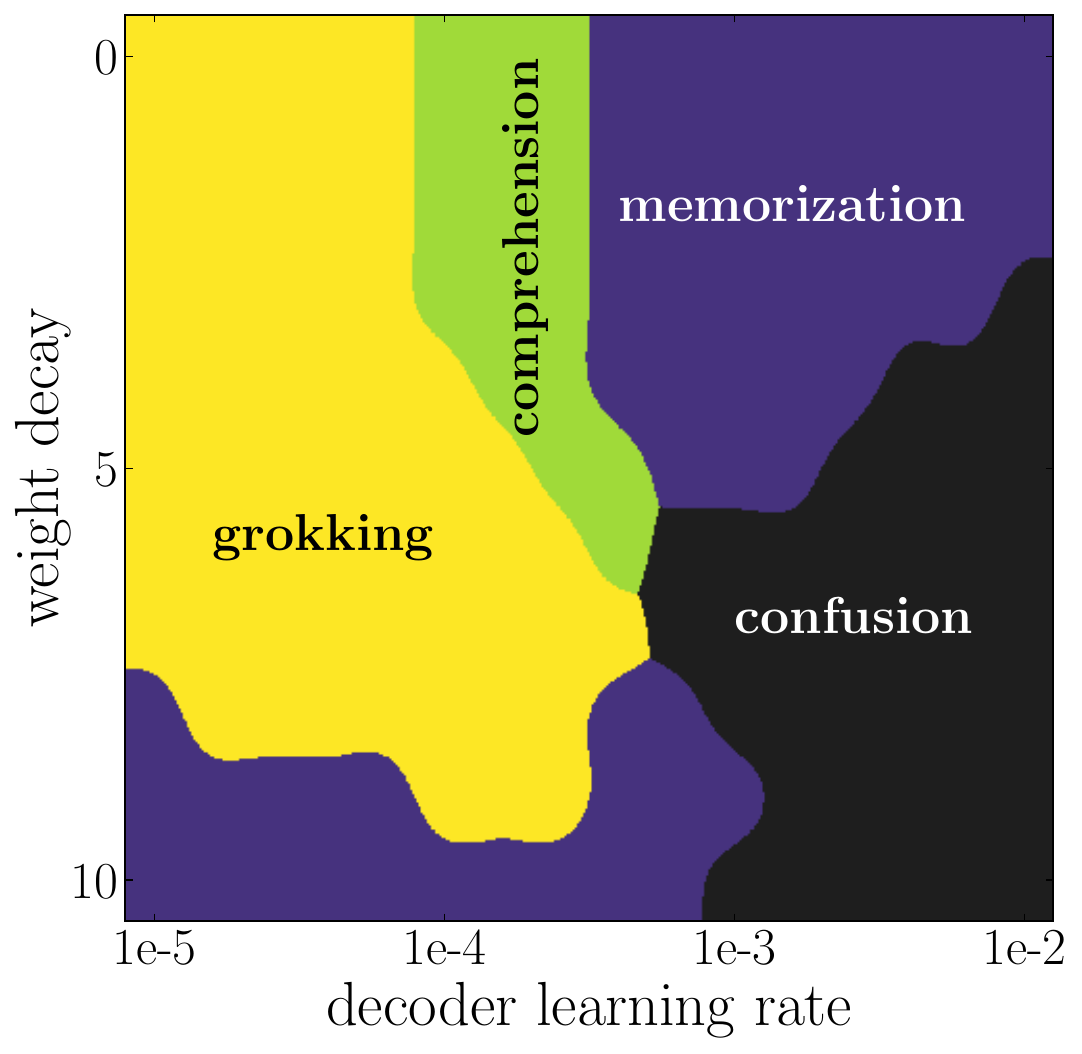}
    \caption{Addition, regression}
    \end{subfigure}
    \begin{subfigure}[]{0.35\textwidth}
    \includegraphics[width=\linewidth]{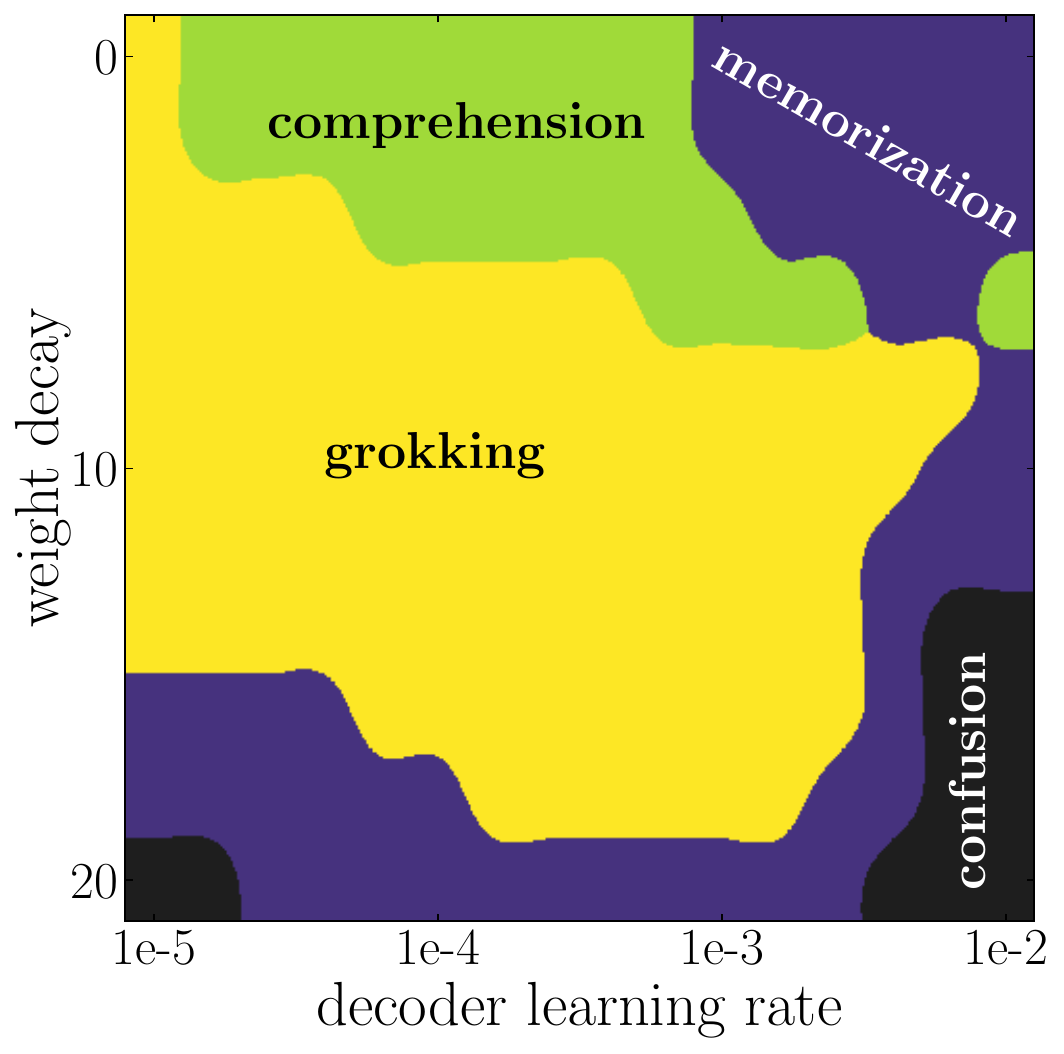}
    \caption{Addition, classification}
    \end{subfigure}
    \begin{subfigure}[]{0.35\textwidth}
    \includegraphics[width=\linewidth]{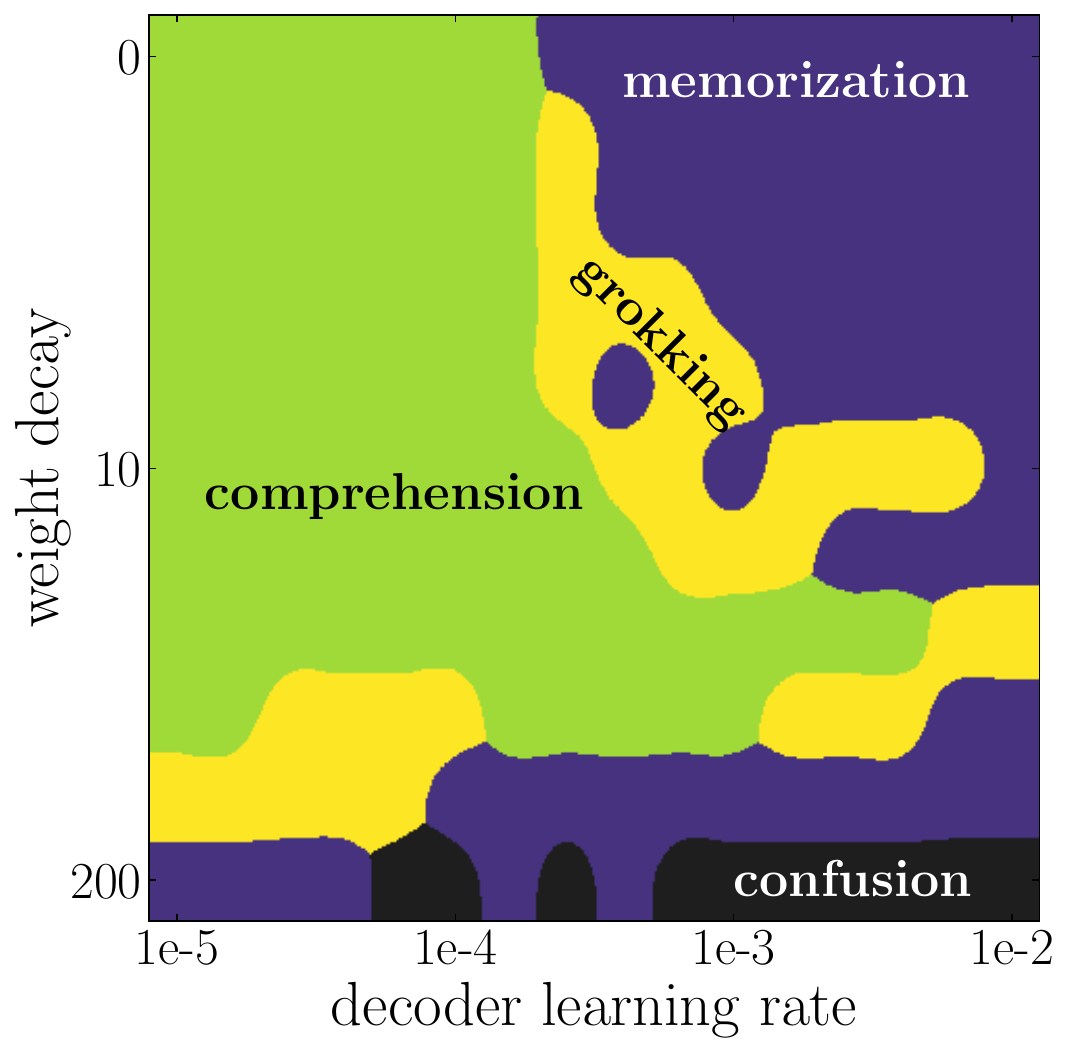}
    \caption{Permutation, regression}
    \end{subfigure}
    \caption{Phase diagrams of learning for the addition group and the permutation group. (a) shows the competition between representation and decoder. (b)(c)(d): each phase diagram contains four phases: comprehension, grokking, memorization and confusion, defined in Table \ref{tab:four_phases}. In (b)(c)(d), grokking is sandwiched between comprehension and memorization.
    }
    \label{fig:grokking_pd}
\end{figure}

{\bf Competition between representation learning and decoder overfitting} In the regression setup of the addition dataset, we show how the competition between representation learning and decoder learning (which depend on both learning rate and weight decay, among other things) lead to different learning phases in Figure~\ref{fig:grokking_pd} (a). As expected, a fast decoder coupled with slow representation learning (bottom right) lead to memorization. In the opposite extreme, although an extremely slow decoder coupled with fast representation learning (top left) will generalize in the end, the generalization time is long due to the inefficient decoder training. The ideal phase (comprehension) requires representation learning to be faster, but not too much, than the decoder.

Drawing from an analogy to physical systems, one can think of embedding vectors as a group of particles. In our effective theory from Section~\ref{sec:effective_theory}, the dynamics of the particles are described \emph{only} by their relative positions, in that sense, structure forms mainly due to inter-particle interactions (in reality, these interactions are mediated by the decoder and the loss).
The decoder plays the role of an environment exerting external forces on the embeddings. If the magnitude of the external forces are small/large one can expect better/worse representations.

{\bf Universality of phase diagrams} We fix the embedding learning rate to be $10^{-3}$ and sweep instead decoder weight decay in Figure~\ref{fig:grokking_pd} (b)(c)(d). The phase diagrams correspond to addition regression (b), addition classification (c) and permutation regression (d), respectively. Common phenomena emerge from these different tasks: (i) they all include four phases; (ii) The top right corner (a fast and capable decoder) is the memorization phase; (iii) the bottom right corner (a fast and simple decoder) is the confusion phase; (iv) grokking  is sandwiched between comprehension and memorization, which seems to imply that it is an undesirable phase that stems from improperly tuned hyperparameters. 

\subsection{Beyond the toy model} 
\label{sec:beyond_toy}
We conjecture that many of the principles which we saw dictate the training dynamics in the toy model also apply more generally. Below, we will see how our framework generalizes to transformer architectures for the task of addition modulo $p$, a minimal reproducible example of the original grokking paper \cite{power2022grokking}. 

We first encode $p=53$ integers into 256D learnable embeddings, then pass two integers to a decoder-only transformer architecture. For simplicity, we do not encode the operation symbols here. The outputs from the last layer are concatenated and passed to a linear layer for classification. Training both the encoder and the decoder with the same optimizer (i.e., with the same hyperparameters) leads to the grokking phenomenon. 
Generalization appears much earlier once we lower the effective decoder capacity with weight decay 
(full phase diagram in Figure \ref{fig:degrok}).

\begin{figure}[h]
    \begin{subfigure}{.3\textwidth}
      \centering
      \includegraphics[width=\linewidth]{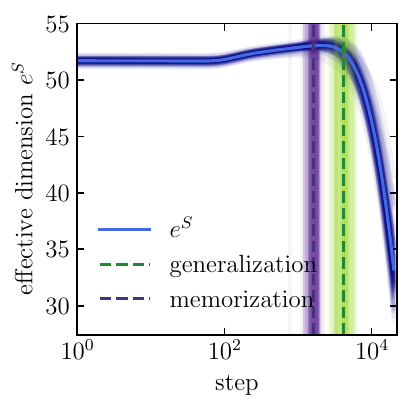}
      \label{fig:transformer-entropy}
      \vspace{-15pt}%
    \end{subfigure}
    \begin{subfigure}{.3\textwidth}
      \centering
      \includegraphics[width=\linewidth]{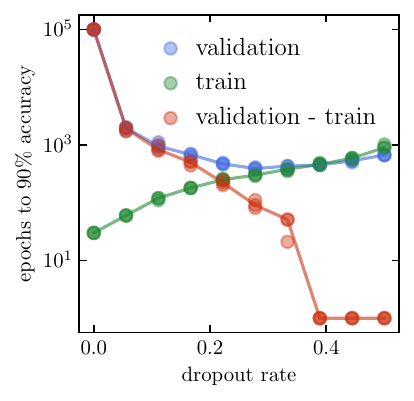}
      \label{fig:transformer-dropout}
      \vspace{-15pt}%
    \end{subfigure}%
    \centering
    \begin{subfigure}{.31\textwidth}
      \centering
      \includegraphics[width=\linewidth]{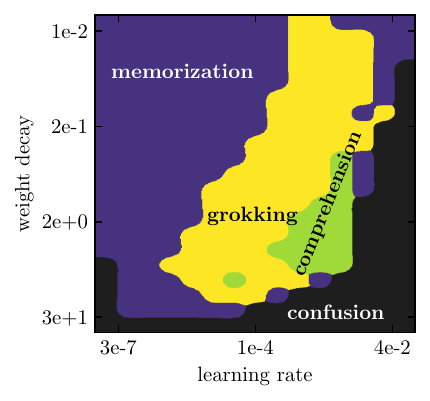}
      \label{fig:transformer-phase}
      \vspace{-15pt}%
    \end{subfigure}
    \caption{
    Left: Evolution of the effective dimension of the embeddings (defined as the exponential of the entropy) during training and evaluated over 100 seeds.
    Center: Effect of dropout on speeding up generalization.
    Right: Phase diagram of the transformer architecture. A scan is performed over the weight decay and learning rate of the decoder while the learning rate of the embeddings is kept fixed at $10^{-3}$ (with zero weight decay).
    }
    \label{fig:degrok}
\end{figure}

Early on, the model is able to perfectly fit the training set while having no generalization. We study the embeddings at different training times and find that neither PCA (shown in Figure~\ref{fig:ring}) nor t-SNE (not shown here) reveal any structure. Eventually, validation accuracy starts to increase, and perfect generalization coincides with the PCA projecting the embeddings into a circle in 2D. Of course, no choice of dimensionality reduction is guaranteed to find any structure, and thus, it is challenging to show explicitly that generalization only occurs when a structure exists. Nevertheless, the fact that, when coupled with the implicit regularization of the optimizer for sparse solutions, such a clear structure appears in a simple PCA so quickly at generalization time suggests that our analysis in the toy setting is applicable here as well. This is also seen in the evolution of the entropy of the explained variance ratio in the PCA of the embeddings (defined as $S = - \sum_i \sigma_i \log\sigma_i$ where $\sigma_i$ is the fractional variance explained by the $i$th principal component). As seen in Figure~\ref{fig:degrok}, the entropy increases up to generalization time then decreases drastically afterwards which would be consistent with the conjecture that generalization occurs when a low-dimensional structure is discovered. The decoder then primarily relies on the information in this low-dimensional manifold and essentially ``prunes'' the rest of the high-dimensional embedding space. Another interesting insight appears when we project the embeddings at initialization onto the principal axes at the end of training. Some of the structure required for generalization exists before training hinting at a connection with the Lottery Ticket Hypothesis.
See Appendix~\ref{app:lth} for more details.

In Figure \ref{fig:degrok} (right), we show a comparable phase diagram to Figure \ref{fig:grokking_pd} evaluated now in the transformer setting.
Note that, as opposed to the setting in \cite{power2022grokking}, weight decay has only been applied to the decoder and not to the embedding layer.
Contrary to the toy model, a certain amount of weight decay proves beneficial to generalization and speeds it up significantly.
We conjecture that this difference comes from the different embedding dimensions. With a highly over-parameterized setting, a non-zero weight decay gives a crucial incentive to reduce complexity in the decoder and help generalize in fewer steps.
This is subject to further investigation. 
We also explore the effect of dropout layers in the decoder blocks of the transformer.
With a significant dropout rate, the generalization time can be brought down to under $10^3$ steps and the grokking phenomenon vanishes completely.
The overall trend suggests that constraining the decoder with the same tools used to avoid overfitting reduces generalization time and can avoid the grokking phenomenon. This is also observed in an image classification task where we were able to induce grokking. See Appendix~\ref{app:mnist_grok} for more details.

\subsection{Grokking Experiment on MNIST}
\label{app:mnist_grok_main}

\begin{figure}[h!]
    \begin{subfigure}{0.6\textwidth}
      \centering
      \includegraphics[height=2in]{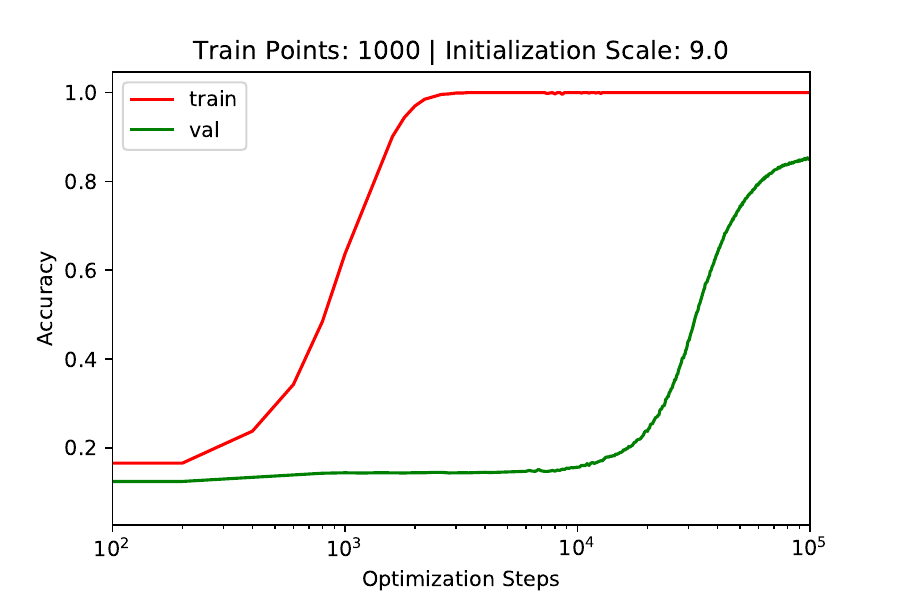}
        \caption{}
        \label{fig:mnist-learning-curve}
    \end{subfigure}
    \begin{subfigure}{0.35\textwidth}
      \centering
      \includegraphics[height=2in]{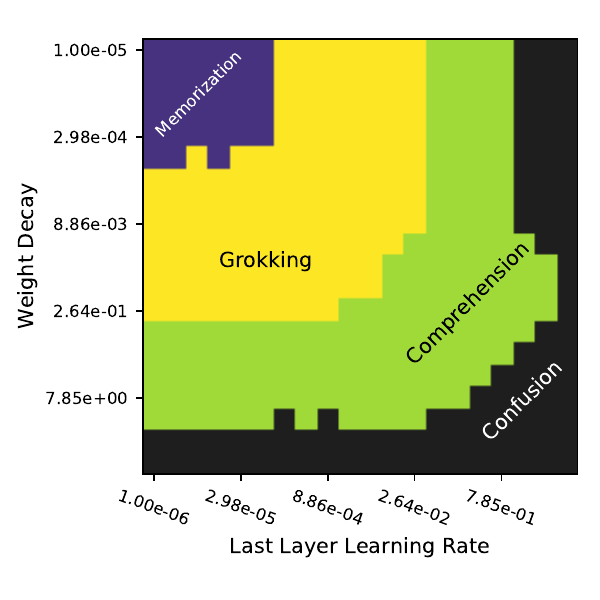}
        \caption{}
        \label{fig:mnist-phase}
    \end{subfigure}
    \caption{Left: Training curves for a run on MNIST, in the setting where we observe grokking. Right: Phase diagram with the four phases of learning dynamics on MNIST.}
    \label{fig:mnist-grok}
\end{figure}

We now demonstrate, for the first time, that grokking (significantly delayed generalization) is a more general phenomenon in machine learning that can occur not only on algorithmic datasets, but also on mainstream benchmark datasets. In particular, we exhibit grokking on MNIST in Figure~\ref{fig:mnist-grok} and demonstrate that we can control grokking by varying optimization hyperparameters. More details on the experimental setup are in Appendix~\ref{app:mnist_grok}.
\section{Related work}\label{sec:related_works}
Relatively few works have analyzed the phenomenon of grokking. \cite{neel2022} describe the circuit that transformers use to perform modular addition, track its formation over training, and broadly suggest that grokking is related to the phenomenon of ``phase changes'' in neural network training. \cite{beren2022grokking, rohin2021grokking} provided earlier speculative, informal conjectures on grokking~\cite{beren2022grokking, rohin2021grokking}. 
Our work is related to the following broad research directions:


{\bf Learning mathematical structures} ~\cite{hoshen2016visual} trains a neural network to learn arithmetic operation from pictures of digits, but they do not observe grokking due to their abundant training data. Beyond arithmetic relations, machine learning has been applied to learn other mathematical structures, including geometry~\cite{he2021machine}, knot theory~\cite{gukov2021learning} and group theory~\cite{davies2021advancing}.

{\bf Double descent} Grokking is somewhat reminiscent of the phenomena of ``epoch-wise'' \emph{double descent}~\cite{nakkiran2021deep}, where generalization can improve after a period of overfitting. \cite{nakkiran2020optimal} find that regularization can mitigate double descent, similar perhaps to how weight decay influences grokking.


{\bf Representation learning} Representation learning lies at the core of machine learning~\cite{bengio2013representation,ouali2020overview,grill2020bootstrap,le2020contrastive}. Representation quality is usually measured by (perhaps vague) semantic meanings or performance on downstream tasks. In our study, the simplicity of arithmetic datasets allows us to define representation quality and study evolution of representations in a quantitative way.

{\bf Physics of learning} Physics-inspired tools have proved to be useful in understanding deep learning from a theoretical perspective. These tools include effective theories~\cite{halverson2021neural, roberts2021principles}, conservation laws~\cite{kunin2020neural} and free energy principle~\cite{gao2020free}. In addition, statistical physics has been identified as a powerful tool in studying generalization in neural networks~\cite{gerace2020generalisation, pezeshki2022multi,pmlr-v145-goldt22a,kuhn1993statistical}. 
Our work connects a low-level understanding of models with their high-level performance. 
In a recent work, researchers at Anthropic~\cite{olsson2022context}, connect a sudden decrease in loss during training with the emergence of \emph{induction heads} within their models. They analogize their work to \emph{statistical physics}, since it bridges a ``microscopic'', mechanistic understanding of networks with ``macroscopic'' facts about overall model performance. 

\section{Conclusion}\label{sec:conclusions}
We have shown how, in both toy models and general settings, that representation enables generalization when it reflects structure in the data. We developed an effective theory of representation learning dynamics (in a toy setting) which predicts the critical dependence of learning on the training data fraction.  
We then presented four learning phases (comprehension, grokking, memorization and confusion) which depend on the decoder capacity and learning speed (given by, among other things, learning rate and weight decay) in decoder-only architectures. While we have mostly focused on a toy model, we find preliminary evidence that our results generalize to the setting of~\cite{power2022grokking}.


Our work can be viewed as a step towards a \emph{statistical physics of deep learning}, connecting the ``microphysics'' of low-level network dynamics with the ``thermodynamics'' of high-level model behavior. We view the application of theoretical tools from physics, such as effective theories~\cite{PDLT-2022}, to be a rich area for further work. The broader impact of such work, if successful, could be to make models more transparent and predictable~\cite{olsson2022context, ganguli2022predictability, moreisdifferentforai}, crucial to the task of ensuring the safety of advanced AI systems.

\bibliography{arithmetic}

\begin{thebibliography}{34}
\providecommand{\natexlab}[1]{#1}
\providecommand{\url}[1]{\texttt{#1}}
\expandafter\ifx\csname urlstyle\endcsname\relax
  \providecommand{\doi}[1]{doi: #1}\else
  \providecommand{\doi}{doi: \begingroup \urlstyle{rm}\Url}\fi

\bibitem[Power et~al.(2022)Power, Burda, Edwards, Babuschkin, and
  Misra]{power2022grokking}
Alethea Power, Yuri Burda, Harri Edwards, Igor Babuschkin, and Vedant Misra.
\newblock Grokking: Generalization beyond overfitting on small algorithmic
  datasets.
\newblock \emph{arXiv preprint arXiv:2201.02177}, 2022.

\bibitem[Nanda and Lieberum(2022)]{neel2022}
Neel Nanda and Tom Lieberum.
\newblock A mechanistic interpretability analysis of grokking, 2022.
\newblock URL
  \url{https://www.alignmentforum.org/posts/N6WM6hs7RQMKDhYjB/a-mechanistic-interpretability-analysis-of-grokking}.

\bibitem[Millidge(2022)]{beren2022grokking}
Beren Millidge.
\newblock {Grokking 'grokking'}.
\newblock \url{https://beren.io/2022-01-11-Grokking-Grokking/}, 2022.

\bibitem[Shah(2021)]{rohin2021grokking}
Rohin Shah.
\newblock {Alignment Newsletter \#159}.
\newblock
  \url{https://www.alignmentforum.org/posts/zvWqPmQasssaAWkrj/an-159-building-agents-that-know-how-to-experiment-by#DEEP_LEARNING_},
  2021.

\bibitem[Hoshen and Peleg(2016)]{hoshen2016visual}
Yedid Hoshen and Shmuel Peleg.
\newblock Visual learning of arithmetic operation.
\newblock In \emph{AAAI}, 2016.

\bibitem[He(2021)]{he2021machine}
Yang-Hui He.
\newblock Machine-learning mathematical structures.
\newblock \emph{arXiv preprint arXiv:2101.06317}, 2021.

\bibitem[Gukov et~al.(2021)Gukov, Halverson, Ruehle, and
  Su{\l}kowski]{gukov2021learning}
Sergei Gukov, James Halverson, Fabian Ruehle, and Piotr Su{\l}kowski.
\newblock Learning to unknot.
\newblock \emph{Machine Learning: Science and Technology}, 2\penalty0
  (2):\penalty0 025035, 2021.

\bibitem[Davies et~al.(2021)Davies, Veli{\v{c}}kovi{\'c}, Buesing, Blackwell,
  Zheng, Toma{\v{s}}ev, Tanburn, Battaglia, Blundell, Juh{\'a}sz,
  et~al.]{davies2021advancing}
Alex Davies, Petar Veli{\v{c}}kovi{\'c}, Lars Buesing, Sam Blackwell, Daniel
  Zheng, Nenad Toma{\v{s}}ev, Richard Tanburn, Peter Battaglia, Charles
  Blundell, Andr{\'a}s Juh{\'a}sz, et~al.
\newblock Advancing mathematics by guiding human intuition with ai.
\newblock \emph{Nature}, 600\penalty0 (7887):\penalty0 70--74, 2021.

\bibitem[Nakkiran et~al.(2021)Nakkiran, Kaplun, Bansal, Yang, Barak, and
  Sutskever]{nakkiran2021deep}
Preetum Nakkiran, Gal Kaplun, Yamini Bansal, Tristan Yang, Boaz Barak, and Ilya
  Sutskever.
\newblock Deep double descent: Where bigger models and more data hurt.
\newblock \emph{Journal of Statistical Mechanics: Theory and Experiment},
  2021\penalty0 (12):\penalty0 124003, 2021.

\bibitem[Nakkiran et~al.(2020)Nakkiran, Venkat, Kakade, and
  Ma]{nakkiran2020optimal}
Preetum Nakkiran, Prayaag Venkat, Sham Kakade, and Tengyu Ma.
\newblock Optimal regularization can mitigate double descent.
\newblock \emph{arXiv preprint arXiv:2003.01897}, 2020.

\bibitem[Bengio et~al.(2013)Bengio, Courville, and
  Vincent]{bengio2013representation}
Yoshua Bengio, Aaron Courville, and Pascal Vincent.
\newblock Representation learning: A review and new perspectives.
\newblock \emph{IEEE transactions on pattern analysis and machine
  intelligence}, 35\penalty0 (8):\penalty0 1798--1828, 2013.

\bibitem[Ouali et~al.(2020)Ouali, Hudelot, and Tami]{ouali2020overview}
Yassine Ouali, C{\'e}line Hudelot, and Myriam Tami.
\newblock An overview of deep semi-supervised learning.
\newblock \emph{arXiv preprint arXiv:2006.05278}, 2020.

\bibitem[Grill et~al.(2020)Grill, Strub, Altch{\'e}, Tallec, Richemond,
  Buchatskaya, Doersch, Avila~Pires, Guo, Gheshlaghi~Azar,
  et~al.]{grill2020bootstrap}
Jean-Bastien Grill, Florian Strub, Florent Altch{\'e}, Corentin Tallec, Pierre
  Richemond, Elena Buchatskaya, Carl Doersch, Bernardo Avila~Pires, Zhaohan
  Guo, Mohammad Gheshlaghi~Azar, et~al.
\newblock Bootstrap your own latent-a new approach to self-supervised learning.
\newblock \emph{Advances in Neural Information Processing Systems},
  33:\penalty0 21271--21284, 2020.

\bibitem[Le-Khac et~al.(2020)Le-Khac, Healy, and Smeaton]{le2020contrastive}
Phuc~H Le-Khac, Graham Healy, and Alan~F Smeaton.
\newblock Contrastive representation learning: A framework and review.
\newblock \emph{IEEE Access}, 8:\penalty0 193907--193934, 2020.

\bibitem[Halverson et~al.(2021)Halverson, Maiti, and
  Stoner]{halverson2021neural}
James Halverson, Anindita Maiti, and Keegan Stoner.
\newblock Neural networks and quantum field theory.
\newblock \emph{Machine Learning: Science and Technology}, 2\penalty0
  (3):\penalty0 035002, 2021.

\bibitem[Roberts et~al.(2021)Roberts, Yaida, and Hanin]{roberts2021principles}
Daniel~A Roberts, Sho Yaida, and Boris Hanin.
\newblock The principles of deep learning theory.
\newblock \emph{arXiv preprint arXiv:2106.10165}, 2021.

\bibitem[Kunin et~al.(2020)Kunin, Sagastuy-Brena, Ganguli, Yamins, and
  Tanaka]{kunin2020neural}
Daniel Kunin, Javier Sagastuy-Brena, Surya Ganguli, Daniel~LK Yamins, and
  Hidenori Tanaka.
\newblock Neural mechanics: Symmetry and broken conservation laws in deep
  learning dynamics.
\newblock \emph{arXiv preprint arXiv:2012.04728}, 2020.

\bibitem[Gao and Chaudhari(2020)]{gao2020free}
Yansong Gao and Pratik Chaudhari.
\newblock A free-energy principle for representation learning.
\newblock In \emph{International Conference on Machine Learning}, pages
  3367--3376. PMLR, 2020.

\bibitem[Gerace et~al.(2020)Gerace, Loureiro, Krzakala, M{\'e}zard, and
  Zdeborov{\'a}]{gerace2020generalisation}
Federica Gerace, Bruno Loureiro, Florent Krzakala, Marc M{\'e}zard, and Lenka
  Zdeborov{\'a}.
\newblock Generalisation error in learning with random features and the hidden
  manifold model.
\newblock In \emph{International Conference on Machine Learning}, pages
  3452--3462. PMLR, 2020.

\bibitem[Pezeshki et~al.(2022)Pezeshki, Mitra, Bengio, and
  Lajoie]{pezeshki2022multi}
Mohammad Pezeshki, Amartya Mitra, Yoshua Bengio, and Guillaume Lajoie.
\newblock Multi-scale feature learning dynamics: Insights for double descent.
\newblock In \emph{International Conference on Machine Learning}, pages
  17669--17690. PMLR, 2022.

\bibitem[Goldt et~al.(2022)Goldt, Loureiro, Reeves, Krzakala, Mezard, and
  Zdeborova]{pmlr-v145-goldt22a}
Sebastian Goldt, Bruno Loureiro, Galen Reeves, Florent Krzakala, Marc Mezard,
  and Lenka Zdeborova.
\newblock The gaussian equivalence of generative models for learning with
  shallow neural networks.
\newblock In Joan Bruna, Jan Hesthaven, and Lenka Zdeborova, editors,
  \emph{Proceedings of the 2nd Mathematical and Scientific Machine Learning
  Conference}, volume 145 of \emph{Proceedings of Machine Learning Research},
  pages 426--471. PMLR, 16--19 Aug 2022.
\newblock URL \url{https://proceedings.mlr.press/v145/goldt22a.html}.

\bibitem[Kuhn and Bos(1993)]{kuhn1993statistical}
R~Kuhn and S~Bos.
\newblock Statistical mechanics for neural networks with continuous-time
  dynamics.
\newblock \emph{Journal of Physics A: Mathematical and General}, 26\penalty0
  (4):\penalty0 831, 1993.

\bibitem[Olsson et~al.(2022)Olsson, Elhage, Nanda, Joseph, DasSarma, Henighan,
  Mann, Askell, Bai, Chen, Conerly, Drain, Ganguli, Hatfield-Dodds, Hernandez,
  Johnston, Jones, Kernion, Lovitt, Ndousse, Amodei, Brown, Clark, Kaplan,
  McCandlish, and Olah]{olsson2022context}
Catherine Olsson, Nelson Elhage, Neel Nanda, Nicholas Joseph, Nova DasSarma,
  Tom Henighan, Ben Mann, Amanda Askell, Yuntao Bai, Anna Chen, Tom Conerly,
  Dawn Drain, Deep Ganguli, Zac Hatfield-Dodds, Danny Hernandez, Scott
  Johnston, Andy Jones, Jackson Kernion, Liane Lovitt, Kamal Ndousse, Dario
  Amodei, Tom Brown, Jack Clark, Jared Kaplan, Sam McCandlish, and Chris Olah.
\newblock In-context learning and induction heads.
\newblock \emph{Transformer Circuits Thread}, 2022.
\newblock
  https://transformer-circuits.pub/2022/in-context-learning-and-induction-heads/index.html.

\bibitem[Roberts et~al.(2022)Roberts, Yaida, and Hanin]{PDLT-2022}
Daniel~A. Roberts, Sho Yaida, and Boris Hanin.
\newblock \emph{The Principles of Deep Learning Theory}.
\newblock Cambridge University Press, 2022.
\newblock \url{https://deeplearningtheory.com}.

\bibitem[Ganguli et~al.(2022)Ganguli, Hernandez, Lovitt, DasSarma, Henighan,
  Jones, Joseph, Kernion, Mann, Askell, et~al.]{ganguli2022predictability}
Deep Ganguli, Danny Hernandez, Liane Lovitt, Nova DasSarma, Tom Henighan, Andy
  Jones, Nicholas Joseph, Jackson Kernion, Ben Mann, Amanda Askell, et~al.
\newblock Predictability and surprise in large generative models.
\newblock \emph{arXiv preprint arXiv:2202.07785}, 2022.

\bibitem[Steinhardt(2022)]{moreisdifferentforai}
Jacob Steinhardt.
\newblock {Future ML Systems Will Be Qualitatively Different}.
\newblock
  \url{https://www.lesswrong.com/s/4aARF2ZoBpFZAhbbe/p/pZaPhGg2hmmPwByHc},
  2022.

\bibitem[Zaheer et~al.(2017)Zaheer, Kottur, Ravanbakhsh, Poczos, Salakhutdinov,
  and Smola]{zaheer2017deep}
Manzil Zaheer, Satwik Kottur, Siamak Ravanbakhsh, Barnabas Poczos, Russ~R
  Salakhutdinov, and Alexander~J Smola.
\newblock Deep sets.
\newblock \emph{Advances in neural information processing systems}, 30, 2017.

\bibitem[Papyan et~al.(2020)Papyan, Han, and Donoho]{papyan2020prevalence}
Vardan Papyan, XY~Han, and David~L Donoho.
\newblock Prevalence of neural collapse during the terminal phase of deep
  learning training.
\newblock \emph{Proceedings of the National Academy of Sciences}, 117\penalty0
  (40):\penalty0 24652--24663, 2020.

\bibitem[{Wikipedia contributors}(2022)]{enwiki:1091431454}
{Wikipedia contributors}.
\newblock Thomson problem --- {Wikipedia}{,} the free encyclopedia.
\newblock
  \url{https://en.wikipedia.org/w/index.php?title=Thomson_problem&oldid=1091431454},
  2022.
\newblock [Online; accessed 29-July-2022].

\bibitem[Chen and He(2021)]{chen2021exploring}
Xinlei Chen and Kaiming He.
\newblock Exploring simple siamese representation learning.
\newblock In \emph{Proceedings of the IEEE/CVF Conference on Computer Vision
  and Pattern Recognition}, pages 15750--15758, 2021.

\bibitem[Xu et~al.(2019)Xu, Zhang, and Xiao]{xu2019training}
Zhi-Qin~John Xu, Yaoyu Zhang, and Yanyang Xiao.
\newblock Training behavior of deep neural network in frequency domain.
\newblock In \emph{International Conference on Neural Information Processing},
  pages 264--274. Springer, 2019.

\bibitem[Zhang et~al.(2020)Zhang, Xu, Luo, and Ma]{zhang2020type}
Yaoyu Zhang, Zhi-Qin~John Xu, Tao Luo, and Zheng Ma.
\newblock A type of generalization error induced by initialization in deep
  neural networks.
\newblock In \emph{Mathematical and Scientific Machine Learning}, pages
  144--164. PMLR, 2020.

\bibitem[Liu et~al.(2022)Liu, Michaud, and Tegmark]{liu2022omnigrok}
Ziming Liu, Eric~J. Michaud, and Max Tegmark.
\newblock Omnigrok: Grokking beyond algorithmic data, 2022.

\bibitem[Woodworth et~al.(2020)Woodworth, Gunasekar, Lee, Moroshko, Savarese,
  Golan, Soudry, and Srebro]{pmlr-v125-woodworth20a}
Blake Woodworth, Suriya Gunasekar, Jason~D. Lee, Edward Moroshko, Pedro
  Savarese, Itay Golan, Daniel Soudry, and Nathan Srebro.
\newblock Kernel and rich regimes in overparametrized models.
\newblock In Jacob Abernethy and Shivani Agarwal, editors, \emph{Proceedings of
  Thirty Third Conference on Learning Theory}, volume 125 of \emph{Proceedings
  of Machine Learning Research}, pages 3635--3673. PMLR, 09--12 Jul 2020.
\newblock URL \url{https://proceedings.mlr.press/v125/woodworth20a.html}.

\end{thebibliography}

\section*{Checklist}

\begin{enumerate}

\item For all authors...
\begin{enumerate}
  \item Do the main claims made in the abstract and introduction accurately reflect the paper's contributions and scope?
    \answerYes{}
  \item Did you describe the limitations of your work?
    \answerYes{}
  \item Did you discuss any potential negative societal impacts of your work?
    \answerNA{}
  \item Have you read the ethics review guidelines and ensured that your paper conforms to them?
    \answerYes{}
\end{enumerate}

\item If you are including theoretical results...
\begin{enumerate}
  \item Did you state the full set of assumptions of all theoretical results?
    \answerYes{}
        \item Did you include complete proofs of all theoretical results?
    \answerYes{}
\end{enumerate}

\item If you ran experiments...
\begin{enumerate}
  \item Did you include the code, data, and instructions needed to reproduce the main experimental results (either in the supplemental material or as a URL)?
    \answerYes{}
  \item Did you specify all the training details (e.g., data splits, hyperparameters, how they were chosen)?
    \answerYes{}
        \item Did you report error bars (e.g., with respect to the random seed after running experiments multiple times)?
    \answerYes{}
        \item Did you include the total amount of compute and the type of resources used (e.g., type of GPUs, internal cluster, or cloud provider)?
    \answerYes{All experiments were run on a workstation with two NVIDIA A6000 GPUs within a few days.}
\end{enumerate}

\item If you are using existing assets (e.g., code, data, models) or curating/releasing new assets...
\begin{enumerate}
  \item If your work uses existing assets, did you cite the creators?
    \answerNA{}
  \item Did you mention the license of the assets?
    \answerNA{}
  \item Did you include any new assets either in the supplemental material or as a URL?
    \answerNA{}
  \item Did you discuss whether and how consent was obtained from people whose data you're using/curating?
    \answerNA{}
  \item Did you discuss whether the data you are using/curating contains personally identifiable information or offensive content?
    \answerNA{}
\end{enumerate}

\item If you used crowdsourcing or conducted research with human subjects...
\begin{enumerate}
  \item Did you include the full text of instructions given to participants and screenshots, if applicable?
    \answerNA{}
  \item Did you describe any potential participant risks, with links to Institutional Review Board (IRB) approvals, if applicable?
    \answerNA{}
  \item Did you include the estimated hourly wage paid to participants and the total amount spent on participant compensation?
    \answerNA{}
\end{enumerate}

\end{enumerate}

\newpage
\appendix

{\huge Appendix}

\section{Definitions of the phases of learning}
\label{app:defs_table}
\begin{table}[h]
  \caption{Definitions of the four phases of learning}
  \label{tab:four_phases}
  \centering
  \begin{tabular}{cccc}
    \toprule
    & \multicolumn{3}{c}{criteria}\\ \cmidrule(r){2-4}
    Phase  & \makecell{training acc > 90\% \\within $10^5$ steps} & \makecell{validation acc > 90\% \\within $10^5$ steps} & \makecell{step(validation acc>90\%)\\$-$step(training acc>90\%)<$10^3$} \\
    \midrule
{\bf Comprehension}   &  Yes & Yes & Yes \\
    {\bf Grokking} & Yes & Yes & No \\
    {\bf Memorization} & Yes & No & Not  Applicable \\
    {\bf Confusion} & No & No & Not Applicable\\ \bottomrule
  \end{tabular}
\end{table}

\section{Applicability of our toy setting}\label{app:applicability}

In the main paper, we focused on the toy setting with (1) the addition dataset and (2) the addition operation hard coded in the decoder. Although both simplifications appear to have quite limited applicability, we argue below that the analysis of the toy setting can actually apply to all Abelian groups.

{\bf The addition dataset is the building block of all Abelian groups}
A cyclic group is a group that is generated by a single element. A finite cyclic group with order $n$ is $C_n=\{e, g, g^2,\cdots, g^{n-1}\}$ where $e$ is the identify element and $g$ is the generator and $g^i=g^j$ whenever $i=j\ ({\rm mod\ }n)$. The modulo addition and $\{0,1,\cdots,n-1\}$ form a cyclic group with $e=0$ and $g$ can be any number $q$ coprime to $n$ such that $(q,n)=1$. Since algorithmic datasets contain only symbolic but no arithmetic information, the datasets of modulo addition could apply to all other cyclic groups, e.g., modulo multiplication and discrete rotation groups in 2D.

Although not all Abelian groups are cyclic, a finite Abelian group $G$ can be always decomposed into a direct product of $k$ cyclic groups $G=C_{n_1}\times C_{n_2}\cdots C_{n_k}$. So after training $k$ neural networks with each handling one cyclic group separately, it is easy to construct a larger neural network that handles the whole Abelian group.

{\bf The addition operation is valid for all Abelian groups}
It is proved in ~\cite{zaheer2017deep} that for a permutation invariant function $f(x_1,x_2,\cdots,x_n)$, there exists $\rho$ and $\phi$ such that
\begin{equation}
    f(x_1,x_2,\cdots,x_n) = \rho[\sum_{i=1}^n\phi(x_i)],
\end{equation}
or $f(x_1,x_2)=\rho(\phi(x_1)+\phi(x_2))$ for $n=2$. Notice that $\phi(x_i)$ corresponds to the embedding vector $\mathbf{E}_i$, $\rho$ corresponds to the decoder. The addition operator naturally emerges from the commutativity of the operator, not restricting the operator itself to be addition. For example, multiplication of two numbers $x_1$ and $x_2$ can be written as $x_1x_2={\rm exp}({\rm ln}(x_1)+{\rm ln}(x_2))$ where $\rho(x)={\rm exp}(x)$ and $\phi(x)={\rm ln}(x)$.

\section{An illustrative example}\label{app:illustration}

We use a concrete case to illustrate why parallelograms lead to generalization (see Figure~\ref{fig:lattice}). For the purpose of illustration, we exploit a curriculum learning setting, where a neural network is fed with a few new samples each time. We will illustrate that, as we have more samples in the training set, the ideal model $\mathcal{M}^*$ (defined in Section \ref{sec:effective_theory}) will arrange the representation $\mat{R}^*$ in a more structured way, i.e., more parallelograms are formed, which helps generalization to unseen validation samples. For simplicity we choose $p=6$.
\begin{itemize}
    \item  $D_1=(0,4)$ and $D_2=(1,3)$ have the same label, so $(0,4,1,3)$ becomes a parallelogram such that $\mat{E}_0+\mat{E}_4=\mat{E}_1+\mat{E}_3\to\mat{E}_3-\mat{E}_0=\mat{E}_4-\mat{E}_1$.  $D_3=(1,5)$ and $D_4=(2,4)$ have the same label, so $(1,5,2,4)$ becomes a parallelogram such that $\mat{E}_1+\mat{E}_5=\mat{E}_2+\mat{E}_4\to\mat{E}_4-\mat{E}_1=\mat{E}_5-\mat{E}_2$. We can derive from the first two equations that $\mat{E}_5-\mat{E}_2=\mat{E}_3-\mat{E}_0\to\mat{E}_0+\mat{E}_5=\mat{E}_2+\mat{E}_3$, which implies that $(0,5,2,3)$ is also a parallelogram (see Figure~\ref{fig:lattice}(a)). This means if $(0,5)$ in training set, our model can predict $(2,3)$ correctly.
    \item  $D_5=(0,2)$ and $D_6=(1,1)$ have the same label, so $\mat{E}_0+\mat{E}_2=2\mat{E}_1$, i.e., $1$ is the middle point of $0$ and $2$ (see Figure~\ref{fig:lattice}(b)). Now we can derive that $2\mat{E}_4=\mat{E}_3+\mat{E}_5$, i.e., $4$ is the middle point of $3$ and $5$. If $(4,4)$ is in the training data, our model can predict $(3,5)$ correctly.
    \item Finally, $D_7=(2,4)$ and $D_8=(3,3)$ have the same label, so $2\mat{E}_3=\mat{E}_2+\mat{E}_4$, i.e., $3$ should be placed at the middle point of $2$ and $4$, ending up Figure~\ref{fig:lattice}(c). This linear structure agrees with the arithmetic structure of $\mathbb{R}$.
\end{itemize}
In summary, although we have $p(p+1)/2=21$ different training samples for $p=6$, we only need $8$ training samples to uniquely determine the perfect linear structure (up to linear transformation). The punchline is: representations lead to generalization.

\begin{figure}[H]
    \centering
    \includegraphics[width=1.0\linewidth]{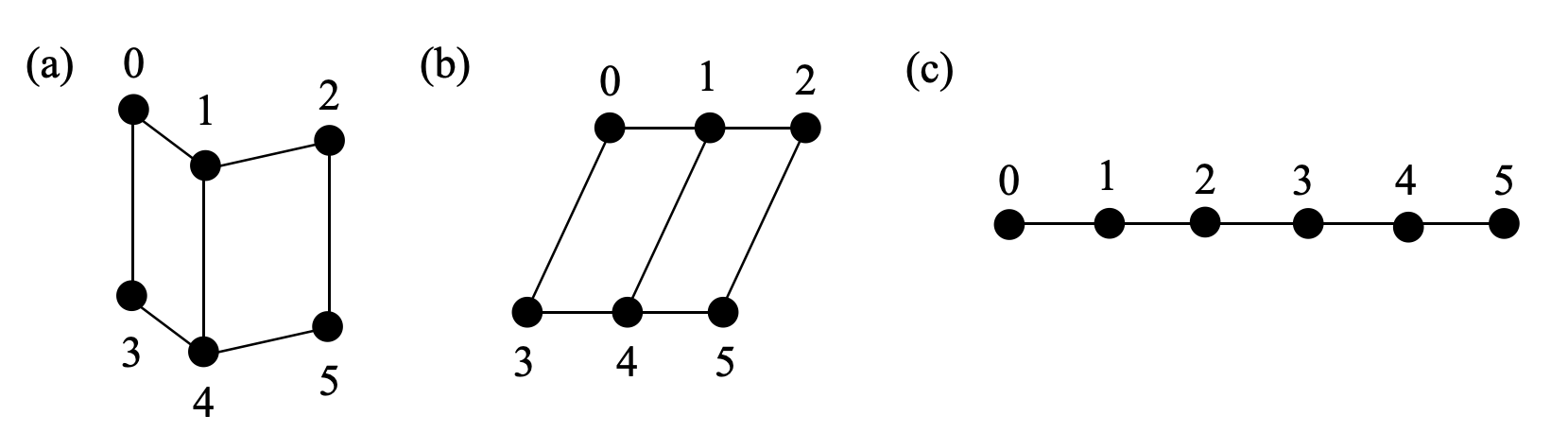}
    \caption{As we include more data in the training set, the (ideal) model is capable of discovering increasingly structured representations (better RQI), from (a) to (b) to (c).}
    \label{fig:lattice}
\end{figure}

\section{Definition of \texorpdfstring{$\widehat{\rm Acc}$}{Acc}}\label{app:acc}
Given a training set $D$ and a representation $\mat{R}$, if $(i,j)$ is a validation sample, can the neural network correctly predict its output, i.e., ${\rm Dec}(\mat{E}_i+\mat{E}_j)=\mat{Y}_{i+j}$? Since neural network has never seen $(i,j)$ in the training set, one possible mechanism of induction is through
\begin{equation}
    {\rm Dec}(\mathbf{E}_i+\mathbf{E}_j)={\rm Dec}(\mathbf{E}_m+\mathbf{E}_n)=\mathbf{Y}_{m+n}(=\mathbf{Y}_{i+j}).
\end{equation}
The first equality ${\rm Dec}(\mathbf{E}_i+\mathbf{E}_j)={\rm Dec}(\mathbf{E}_m+\mathbf{E}_n)$ holds only when $\mathbf{E}_i+\mathbf{E}_j=\mathbf{E}_m+\mathbf{E}_n$ (i.e., $(i,j,m,n)$ is a parallelogram). The second equality $\rm{Dec}(\mathbf{E}_m+\mathbf{E}_n)=\mathbf{Y}_{m+n}$, holds when $(m,n)$ is in the training set, i.e., $(m,n)\in D$, under the zero training loss assumption. Rigorously, given a training set $D$ and a parallelogram set $P$ (which can be calculated from $\mat{R}$), we collect all zero loss samples in an \textit{augmented} training set $\overline{D}$
\begin{equation}\label{eq:augmented_D}
    \overline{D}(D,P)=D\bigcup \{(i,j)|\exists (m,n)\in D, (i,j,m,n)\in P\}.
\end{equation}
Keeping $D$ fixed, a larger $P$ would probably produce a larger $\overline{D}$, i.e., if $P_1\subseteq P_2$, then $\overline{D}(D,P_1)\subseteq \overline{D}(P,P_2)$, which is why in Eq.~(\ref{eq:RQI}) our defined ${\rm RQI}\propto |P|$ gets its name ``representation quality index", because higher RQI normally means better generalization. 
Finally, the expected accuracy from a dataset $D$ and a parallelogram set $P$ is:
\begin{equation}
    \widehat{{\rm Acc}} = \frac{|\overline{D}(D,P)|}{|D_0|},
\end{equation}
which is the estimated accuracy (of the full dataset), and $P=P(\mat{R})$ is defined on the representation after training. On the other hand, accuracy ${\rm Acc}$ can be accessed empirically from trained neural network.
We verified ${\rm Acc}\approx \widehat{\rm Acc}$ in a toy setup (addition dataset $p=10$, 1D embedding space, hard code addition), as shown in Figure~\ref{fig:Acc_DP} (c). Figure~\ref{fig:Acc_DP} (a)(b) show ${\rm Acc}$ and $\widehat{\rm Acc}$ as a function of training set ratio, with each dot corresponding to a different random seed. The dashed red diagonal corresponds to memorization of the training set, and the vertical gap refers to generalization.

Although the agreement is good for 1D embedding vectors, we do not expect such agreement can trivially extend to high dimensional embedding vectors. In high dimensions, our definition of RQI is too restrictive. For example, suppose  we have an embedding space with $N$ dimensions. Although the representation may form a linear structure in the first dimension, the representation can be arbitrary in other $N-1$ dimensions, leading to $\text{RQI}\approx 0$. However, the model may still generalize well if the decoder learns to keep only the useful dimension and drop all other $N-1$ useless dimensions. It would be interesting to investigate how to define an RQI that takes into account the role of decoder in future works.

\section{The gap of a realistic model \texorpdfstring{$\mathcal{M}$}{M} and the ideal model \texorpdfstring{$\mathcal{M}^*$}{M*}}\label{app:ideal-gap}

Realistic models $\mathcal{M}$ usually form fewer number of parallelograms than ideal models $\mathcal{M}^*$. In this section, we analyze the properties of ideal models and calculated ideal RQI and ideal accuracy, which set upper bounds for empirical RQI and accuracy. The upper bound relations are verified via numerical experiments in Figure~\ref{fig:ideal_results}.

Similar to Eq.~(\ref{eq:augmented_D}) where some validation samples can be derived from training samples, we demonstrate how \textit{implicit parallelograms} can be `derived' from explicit ones in $P_0(D)$. The so-called derivation follows a simple geometric argument that: if $A_1B_1$ is equal and parallel to $A_2B_2$, and $A_2B_2$ is equal and parallel to $A_3B_3$, then we can deduce that $A_1B_1$ is equal and parallel to $A_3B_3$ (hence $(A_1,B_2,A_2,B_1)$ is a parallelogram).

Recall that a parallelogram $(i,j,m,n)$ is equivalent to $\mathbf{E}_i+\mathbf{E}_j=\mathbf{E}_m+\mathbf{E}_n$ $(*)$. So we are equivalently asking if equation $(*)$ can be expressed as a linear combination of equations in $A(P_0(D))$. If yes, then $(*)$ is dependent on $A(P_0(D))$ (defined in Eq.~(\ref{eq:A})), i.e., $A(P_0(D))$ and $A(P_0(D)\bigcup{(i,j,m,n)})$ should have the same rank. We augment $P_0(D)$ by adding implicit parallelograms, and denote the augmented parallelogram set as
\begin{equation}\label{eq:PD}
    P(D)=P_0(D)\bigcup \{q\equiv(i,j,m,n)|q\in P_0, {\rm rank}(A(P_0(D)))={\rm rank}(A(P_0(D)\bigcup q))\}.
\end{equation}
We need to emphasize that an assumption behind Eq.~(\ref{eq:PD}) is that we have an ideal model $\mathcal{M}^*$. When the model is not ideal, e.g., when the injectivity of the encoder breaks down, fewer parallelograms are expected to form, i.e.,
\begin{equation}
    P(R)\subseteq P(D).
\end{equation}
The inequality is saying, whenever a parallelogram is formed in the representation after training, the reason is hidden in the training set. This is not a strict argument, but rather a belief that today's neural networks can only copy what datasets (explicitly or implicitly) tell it to do, without any autonomous creativity or intelligence. For simplicity we call this belief \textit{Alexander Principle}. In very rare cases when something lucky happens (e.g., neural networks are initialized at approximate correct weights), Alexander principle may be violated. Alexander principle sets an upper bound for ${\rm RQI}$:
\begin{equation}\label{eq:AP_RQI}
    {\rm RQI}(R)\leq  \frac{|P(D)|}{|P_0|}\equiv \overline{\rm RQI},
\end{equation}
and sets an upper bound for $\widehat{\rm Acc}$:
\begin{equation}\label{eq:AP_Acc}
    \widehat{\rm Acc}\equiv \widehat{\rm Acc}(D,P(R))\leq \widehat{\rm Acc}(D,P(D))\equiv\overline{{\rm Acc}}.
\end{equation}

\begin{figure}
    \centering
    \includegraphics[width=1\linewidth]{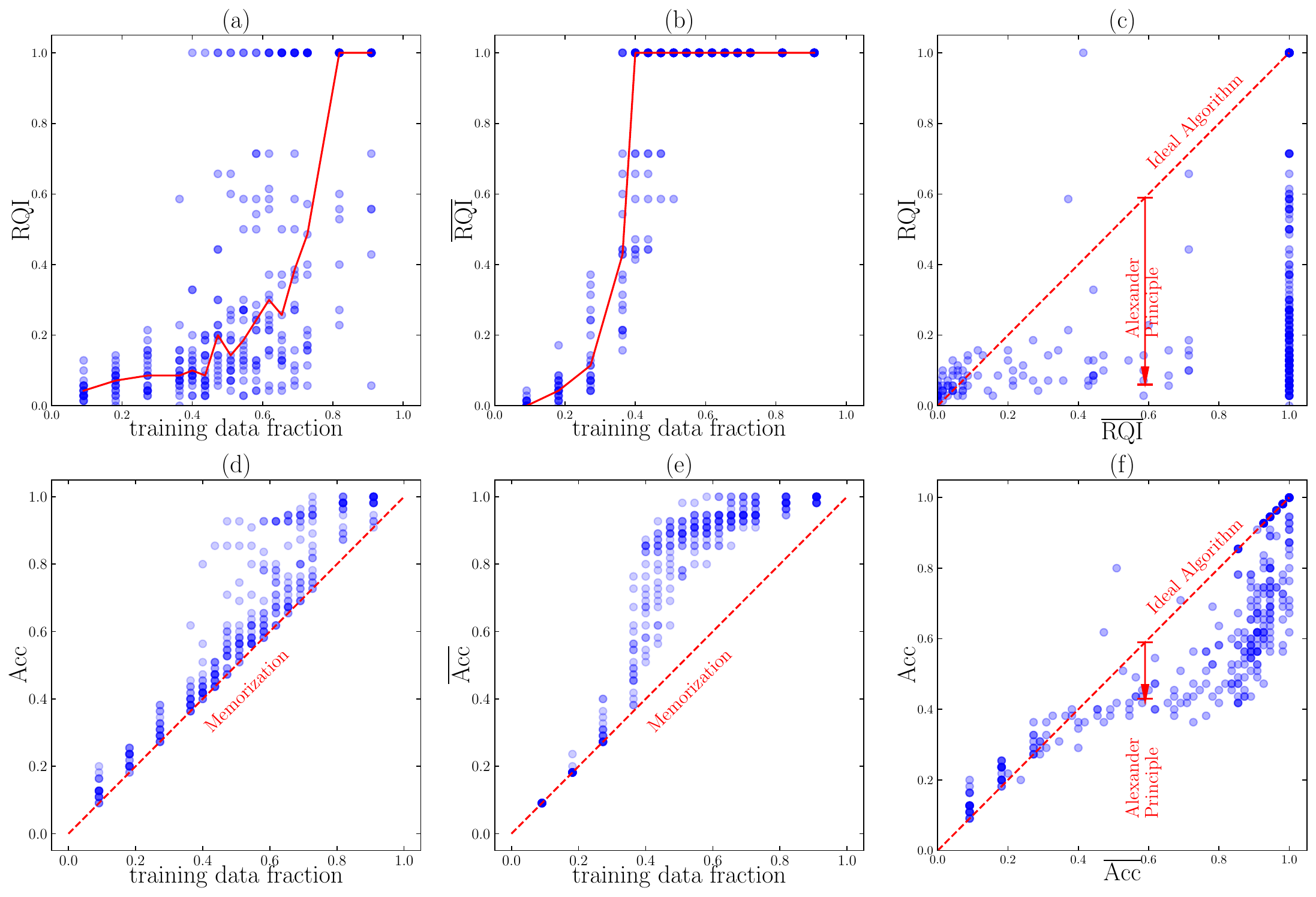}
    \caption{We compare RQI and Acc for an ideal algorithm (with bar) and a realistic algorithm (without bar). In (a)(b)(d)(e), four quantities (RQI, $\overline{\rm RQI}$, Acc, $\overline{\rm Acc}$) as functions of training data fraction are shown. In (c)(f), RQI and Acc of the ideal algorithm sets upper bounds for those of the realistic algorithm.}
    \label{fig:ideal_results}
\end{figure}

In Figure~\ref{fig:ideal_results} (c)(f), we verify Eq.~(\ref{eq:AP_RQI}) and Eq.~(\ref{eq:AP_Acc}). We choose $\delta=0.01$ to compute RQI($R$,$\delta$). We find the trained models are usually far from being ideal, although we already include a few useful tricks proposed in Section \ref{sec:phase_diagram} to enhance representation learning. It would be an interesting future direction to develop better algorithms so that the gap due to Alexander principle can be reduced or even closed. In Figure~\ref{fig:ideal_results} (a)(b)(d)(e), four quantities (RQI, $\overline{\rm RQI}$, Acc, $\overline{\rm Acc}$) as functions of the training data fraction are shown, each dot corresponding to one random seed. It is interesting to note that it is possible to have $\overline{\rm RQI}=1$ only with $<40\%$ training data, i.e., $55\times 0.4= 22$ samples, agreeing with our observation in Section \ref{sec:represenation}.

{\bf Realistic representations} Suppose an ideal model $\mathcal{M}^*$ and a realistic model $\mathcal{M}$ which train on the training set $D$ give the representation $R^*$ and $R$, respectively. What is the relationship between $R$ and $R_*$? Due to the Alexander principle we know $P(R)\subseteq P(D)=P(R^*)$. This means $R^*$ has more parallelograms than $R$, hence $R^*$ has fewer degrees of freedom than $R$. 

We illustrate with the toy case $p=4$. The whole dataset contains $p(p+1)/2=10$ samples, i.e.,
\begin{equation}
    D_0=\{(0,0),(0,1),(0,2),(0,3),(1,1),(1,2),(1,3),(2,2),(2,3),(3,3)\}.
\end{equation}
The parallelogram set contains only three elements, i.e.,
\begin{equation}
    P_0 = \{(0,1,1,2),(0,1,2,3),(1,2,2,3)\},
\end{equation}
Or equivalently the equation set
\begin{equation}
    A_0=\{{\rm A1:\mat{E}_0+\mat{E}_2=2\mat{E}_1,A2:\mat{E}_0+\mat{E}_3=\mat{E}_1+\mat{E}_2,A3:\mat{E}_1+\mat{E}_3=2\mat{E}_2}\}.
\end{equation}
Pictorially, we can split all possible subsets $\{A|A\subseteq A_0\}$ into different levels, each level defined by $|A|$ (the number of elements). A subset $A_1$ in the $i^{\rm th}$ level points an direct arrow to another subset $A_2$ in the $(i+1)^{\rm th}$ level if $A_2\subset A_1$, and we say $A_2$ is a child of $A_1$, and $A_1$ is a parent of $A_2$. Each subset $A$ can determine a representation $R$ with $n(A)$ degrees of freedom. So $R$ should be a descendant of $R_*$, and $n(R_*)\leq n(R)$. Numerically, $n(A)$ is equal to the dimension of the null space of $A$.

\begin{figure}
    \centering
    \includegraphics[width=0.8\linewidth]{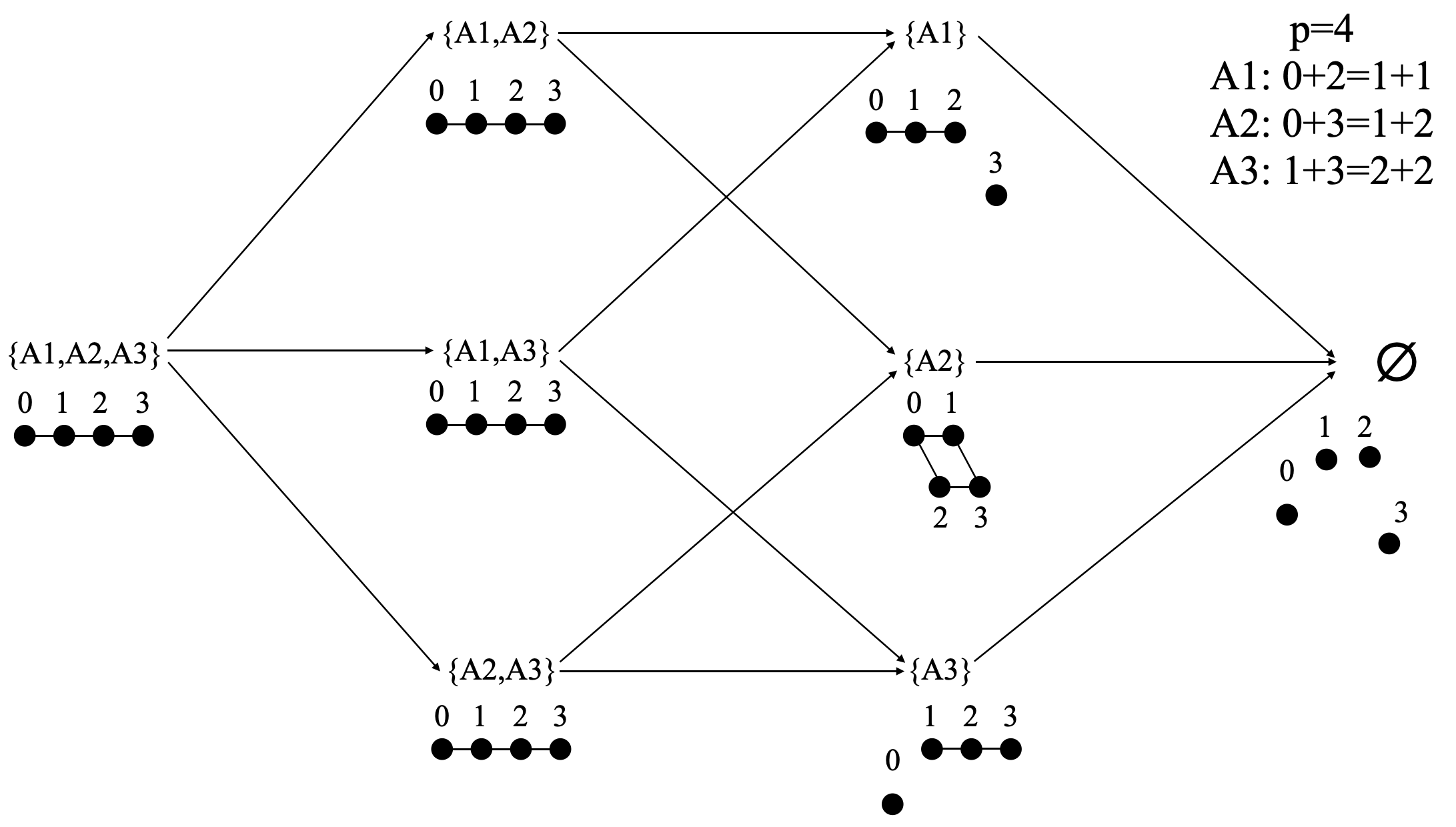}
    \caption{$p=4$ case. Equation set $A$ (or geometrically, representation) has a hierarchy: $a\to b$ means $a$ is a parent of $b$, and $b$ is a child of $a$. A realistic model can only generate representations that are descendants of the representation generated by an ideal model.}
    \label{fig:p4_illustrate}
\end{figure}

\begin{figure}
    \centering
    \includegraphics[width=0.48\linewidth]{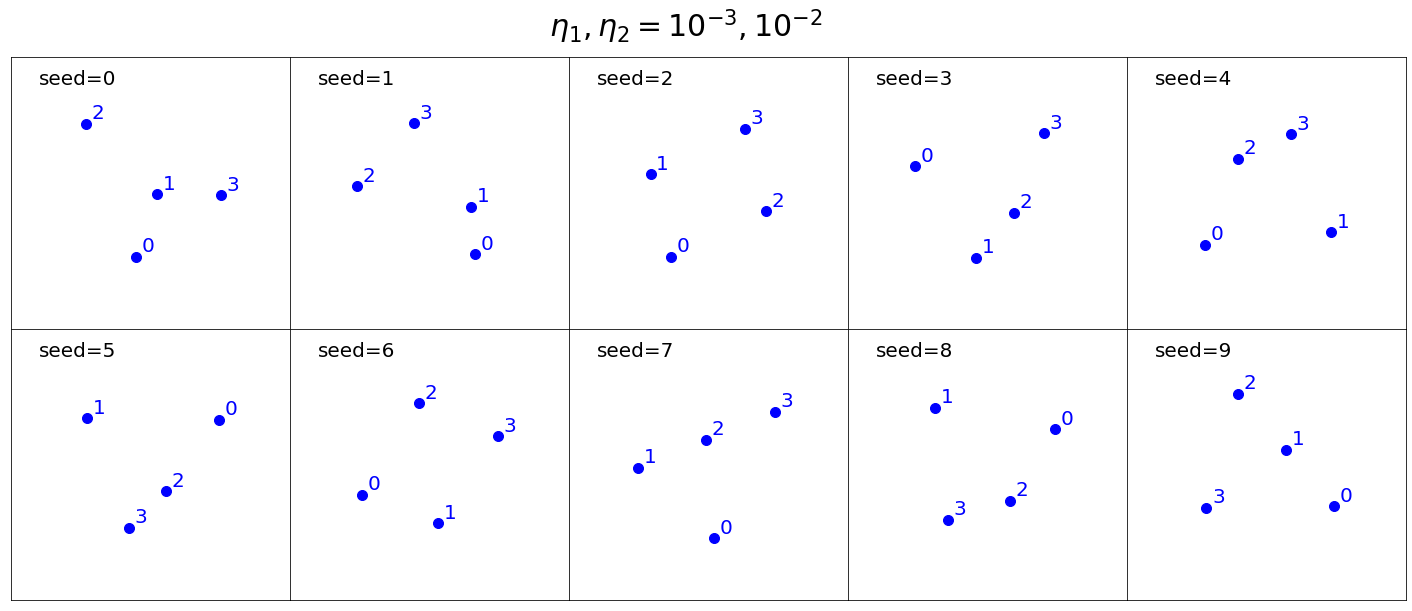}
    \includegraphics[width=0.48\linewidth]{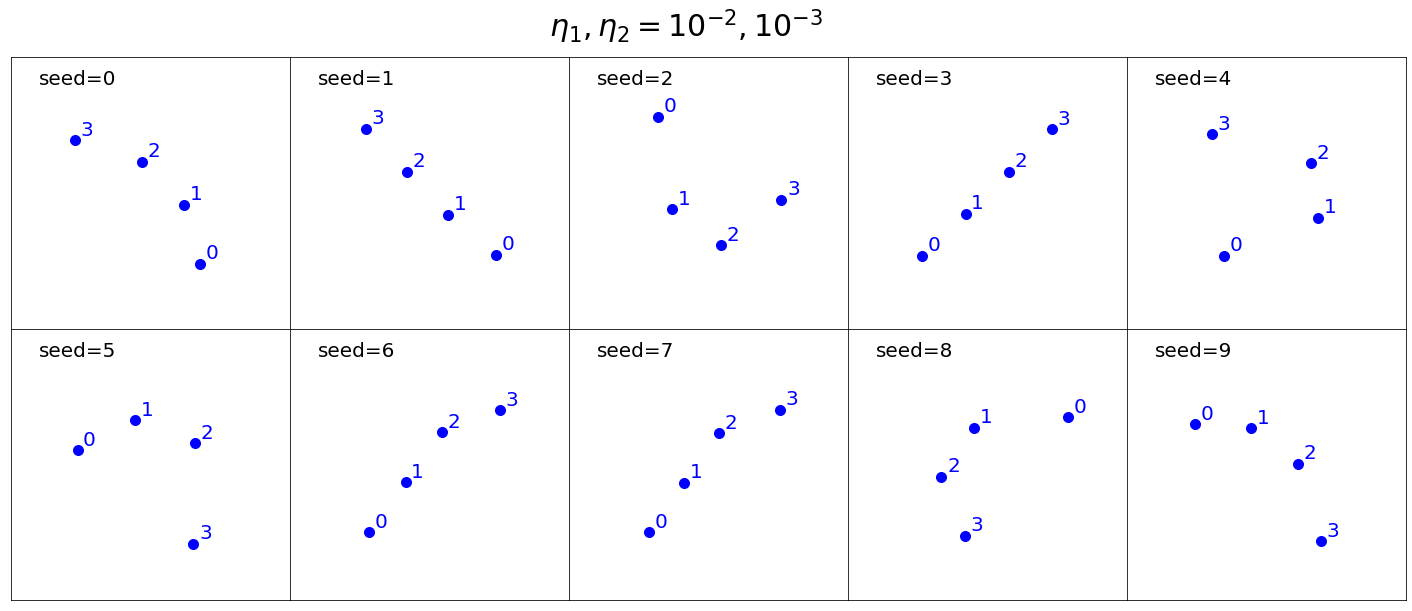}
    \caption{$p=4$ case. Representations obtained from training neural networks are displayed. $\eta_1$ and $\eta_2$ are learning rates of the representation and the decoder, respectively. As described in the main text, $(\eta_1,\eta_2)=(10^{-2},10^{-3})$ (right) is more ideal than $(\eta_1,\eta_2)=(10^{-3},10^{-2})$ (left), thus producing representations containing more parallelograms.}
    \label{fig:p4_exp}
\end{figure}

Suppose we have a training set
\begin{equation}
    D=\{(0,2),(1,1),(0,3),(1,2),(1,3),(2,2)\},
\end{equation}
and correspondingly $P(D)=P_0, A(P)=A_0$. So an ideal model $\mathcal{M}_*$ will have the linear structure $\mat{E}_k=\mat{a}+k\mat{b}$ (see Figure~\ref{fig:p4_illustrate} leftmost). However, a realistic model $\mathcal{M}$ may produce any descendants of the linear structure, depending on various hyperparameters and even random seeds.

In Figure~\ref{fig:p4_exp}, we show our algorithms actually generates all possible representations. We have two settings: (1) fast decoder $(\eta_1,\eta_2)=(10^{-3},10^{-2})$ (Figure~\ref{fig:p4_exp} left), and (2) relatively slow decoder $(\eta_1,\eta_2)=(10^{-2},10^{-3})$ (Figure~\ref{fig:p4_exp}) right). The relatively slow decoder produces better representations (in the sense of higher RQI) than a fast decoder, agreeing with our observation in Section \ref{sec:phase_diagram}.

\section{Conservation laws of the effective theory}\label{app:conservation-law}

Recall that the effective loss function
\begin{equation}\label{eq:l_eff_app}
    \ell_{\text{eff}} = \frac{\ell_0}{Z_0},\quad  \ell_0\equiv\sum_{(i,j,m,n)\in P_0(D)}|\mathbf{E}_i+\mathbf{E}_j-\mathbf{E}_m-\mathbf{E}_n|^2/|P_0(D)|,\quad  Z_0\equiv \sum_{k}|\mathbf{E}_k|^2
\end{equation}
where $\ell_0$ and $Z_0$ are both quadratic functions of $R=\{\mathbf{E}_0,\cdots, \mathbf{E}_{p-1}\}$, and $\ell_{\text{eff}}=0$ remains zero under rescaling and translation $\mathbf{E}_i'=a\mathbf{E}_i+\mat{b}$. We will ignore the $1/|P_0(D)|$ factor in $\ell_0$ since having it is equivalent to rescaing time, which does not affect conservation laws. The representation vector $\mathbf{E}_i$ evolves according to the gradient descent
\begin{equation}\label{eq:Ei_dynamics_app}
    \frac{d\mathbf{E}_i}{dt} = -\frac{\partial \ell_{\text{eff}}}{\partial \mathbf{E}_i}.
\end{equation}
We will prove the following two quantities are conserved:
\begin{equation}
    \mathbf{C}=\sum_k \mathbf{E}_k, \quad 
    Z_0 = \sum_{k} |\mathbf{E}_k|^2.
\end{equation}
Eq.~(\ref{eq:l_eff_app}) and Eq.~(\ref{eq:Ei_dynamics_app}) give
\begin{equation}
    \frac{d\mathbf{E}_i}{dt}=-\frac{\ell_{\text{eff}}}{\partial \mathbf{E}_i}=-\frac{\partial (\frac{\ell_0}{Z_0})}{\partial \mathbf{E}_i}=-\frac{1}{Z_0}\frac{\partial \ell_0}{\partial \mathbf{E}_i} + \frac{\ell_0}{Z_0^2}\frac{\partial Z_0}{\partial \mathbf{E}_i}.
\end{equation}
Then
\begin{align}\label{eq:Z0_conservation}
    \frac{dZ_0}{dt} &= 2\sum_{i} \mathbf{E}_k\cdot\frac{d\mathbf{E}_k}{dt}\\ \nonumber
    &=\frac{2}{Z_0^2}\sum_{i} \mathbf{E}_i\cdot (-Z_0\frac{\partial\ell_0}{\partial \mathbf{E}_k}+2\ell_0\mathbf{E}_k)\\ \nonumber
    &=\frac{2}{Z_0}(-\sum_k \frac{\partial \ell_0}{\partial \mathbf{E}_k}\cdot \mathbf{E}_k+2\ell_0)\\ \nonumber
    &=0.
\end{align}

where the last equation uses the fact that 
\begin{align*}
  \sum_k \frac{\partial \ell_0}{\partial \mathbf{E}_k}\cdot \mathbf{E}_k &= 2 \sum_{k}\sum_{(i,j,m,n)\in P_0(D)}(\mathbf{E}_i+\mathbf{E}_j-\mathbf{E}_m-\mathbf{E}_n)(\delta_{ik}+\delta_{jk}-\delta_{mk}-\delta_{nk}) \cdot \mathbf{E}_k\\
  &= 2 \sum_{(i,j,m,n)\in P_0(D)}(\mathbf{E}_i+\mathbf{E}_j-\mathbf{E}_m-\mathbf{E}_n)\sum_{k}(\delta_{ik}+\delta_{jk}-\delta_{mk}-\delta_{nk}) \cdot \mathbf{E}_k\\
  & =  \sum_{(i,j,m,n)\in P_0(D)}(\mathbf{E}_i+\mathbf{E}_j-\mathbf{E}_m-\mathbf{E}_n) \cdot (\mathbf{E}_{i}+\mathbf{E}_{j}-\mathbf{E}_{m}-\mathbf{E}_{n})\\
  & = 2\ell_0
\end{align*}

The conservation of $Z_0$ prohibits the representation from collapsing to zero. Now that we have demonstrated that $Z_0$ is a conserved quantity, we can also show
\begin{align}
    \frac{d\mathbf{C}}{dt} &= \sum_{k} \frac{d\mathbf{E}_k}{dt}\\ \nonumber
    &= -\frac{1}{Z_0}\sum_k\frac{\partial \ell_0}{\partial \mathbf{E}_k}\\ \nonumber
    &=-\frac{2}{Z_0}\sum_{k}\sum_{(i,j,m,n)\in P_0(D)}(\mathbf{E}_i+\mathbf{E}_j-\mathbf{E}_m-\mathbf{E}_n)(\delta_{ik}+\delta_{jk}-\delta_{mk}-\delta_{nk})\\ \nonumber
    &=\bf0.
\end{align}
The last equality holds because the two summations can be swapped and $\sum_{k}(\delta_{ik}+\delta_{jk}-\delta_{mk}-\delta_{nk})=0$.

\section{More phase diagrams of the toy setup}\label{app:more_toy_pd}
We study another three hyperparameters in the toy setup by showing phase diagrams similar to Figure~\ref{fig:grokking_pd}. The toy setup is: (1) addition without modulo ($p=10$); (2) training/validation is split into 45/10; (3) hard code addition; (4) 1D embedding. In the following experiments, the decoder is an MLP with size 1-200-200-30. The representation and the encoder are optimized with AdamW with different hyperparameters. The learning rate of the representation is $10^{-3}$. We sweep the learning rate of the decoder in range $[10^{-4},10^{-2}]$ as the x axis, and sweep another hyperparameter as the y axis. By default, we use full batch size 45, initialization scale $s=1$ and zero weight decay of representation.

{\bf Batch size} controls the amount of noise in the training dynamics. In Figure~\ref{fig:grokking_pd_bs}, the grokking region appears at the top left of the phase diagram (small decoder learning rate and small batch size). However, large batch size (with small learning rate) leads to comprehension, implying that smaller batch size seems harmful. This makes sense since to get crystals (good structures) in experiments, one needs a freezer which gradually decreases temperature, rather than something perturbing the system with noise.

\begin{figure}
    \centering
    \includegraphics[width=0.42\linewidth]{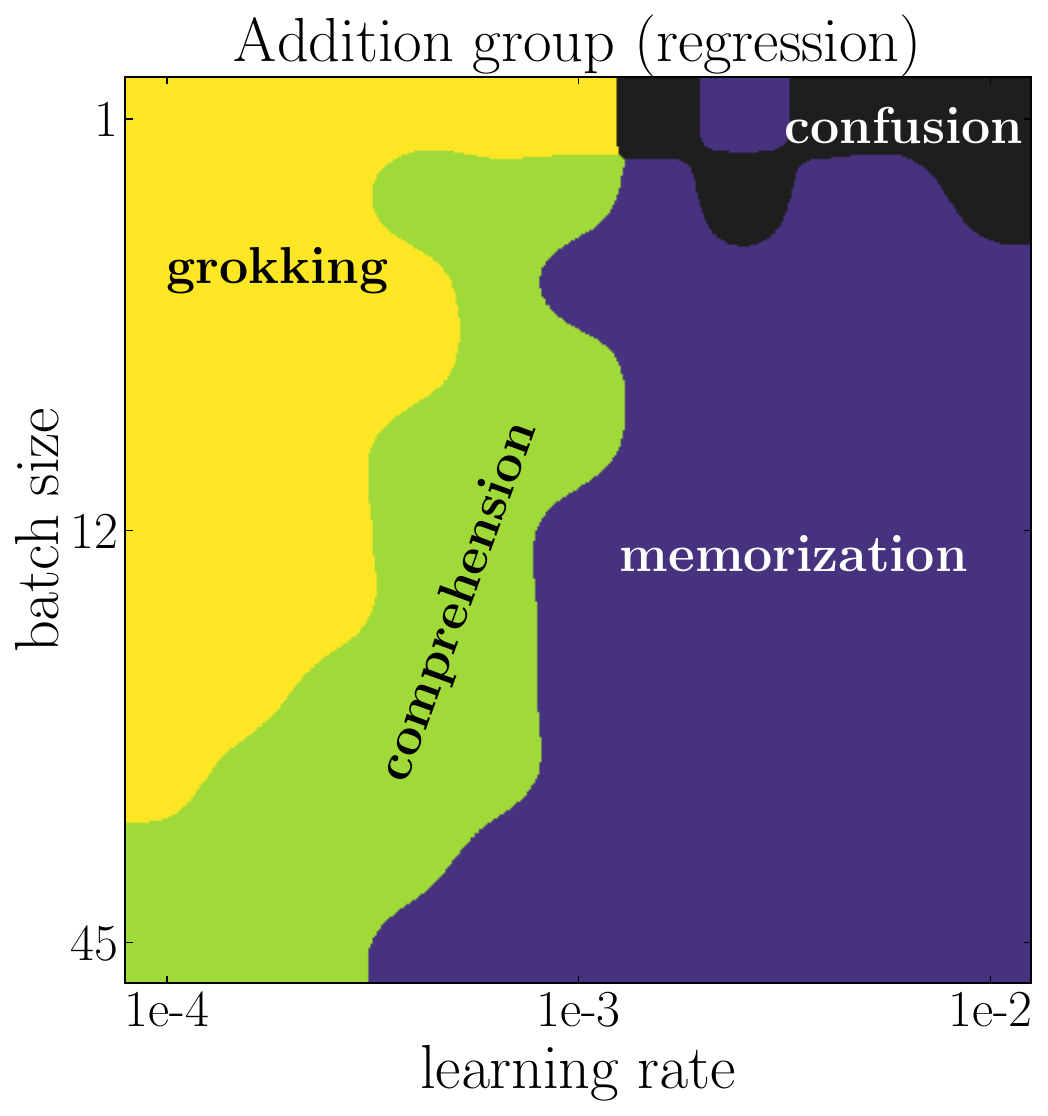}
    \includegraphics[width=0.42\linewidth]{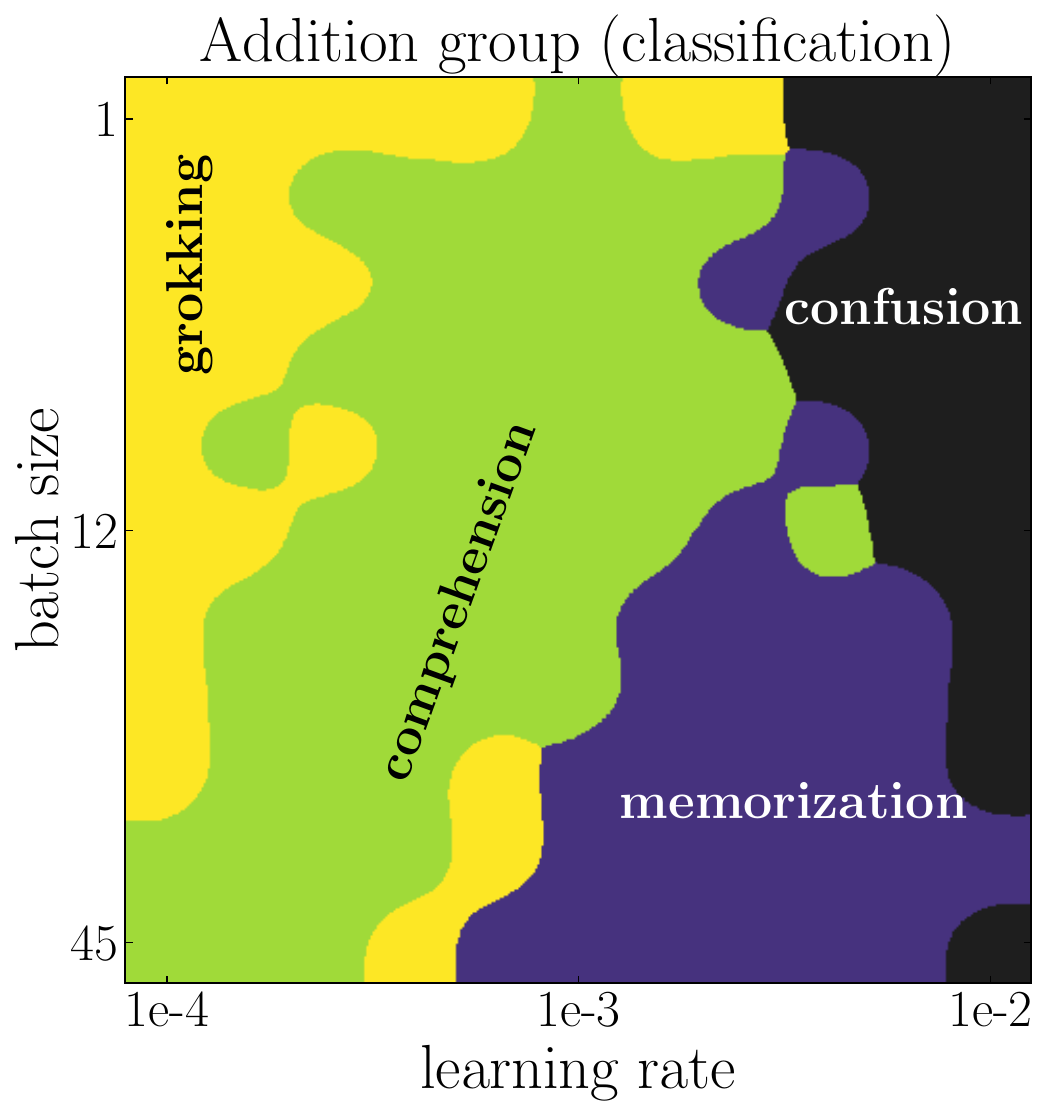}
    \caption{Phase diagrams of  decoder learning rate (x axis) and batch size (y axis) for the addition group (left: regression; right: classification). Small decoder leanrning rate and large batch size (bottom left) lead to comprehension.}
    \label{fig:grokking_pd_bs}
\end{figure}

{\bf Initialization scale} controls distances among embedding vectors at initialization. We initialize components of embedding vectors from independent uniform distribution $U[-s/2,s/2]$ where $s$ is called the initialization scale. Shown in Figure~\ref{fig:grokking_pd_init}, it is beneficial to use a smaller initialization scale. This agrees with the physical intuition that closer particles are more likely to interact and form structures. For example, the distances among molecules in ice are much smaller than distances in gas.
\begin{figure}
    \centering
    \includegraphics[width=0.45\linewidth]{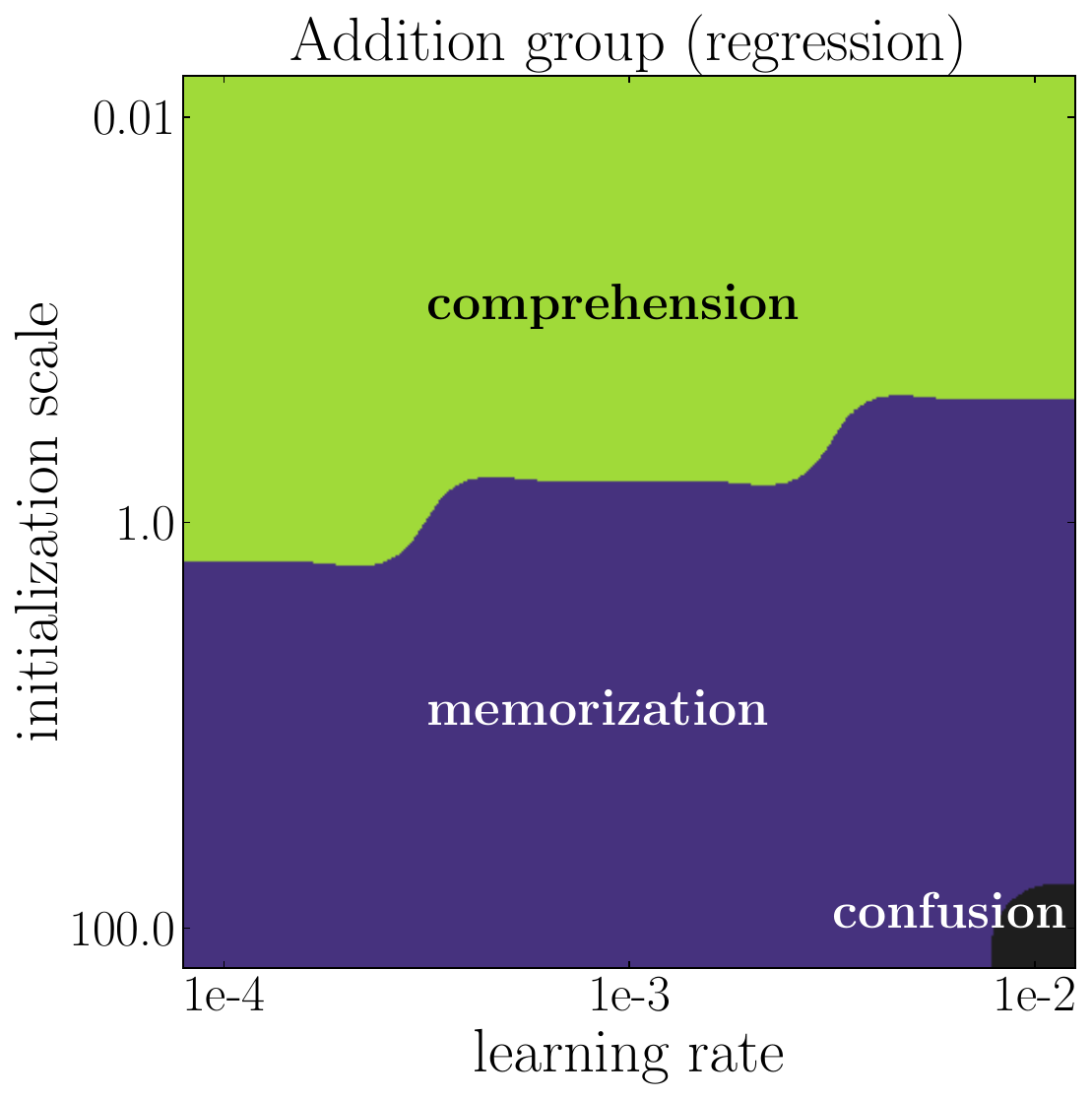}
    \includegraphics[width=0.45\linewidth]{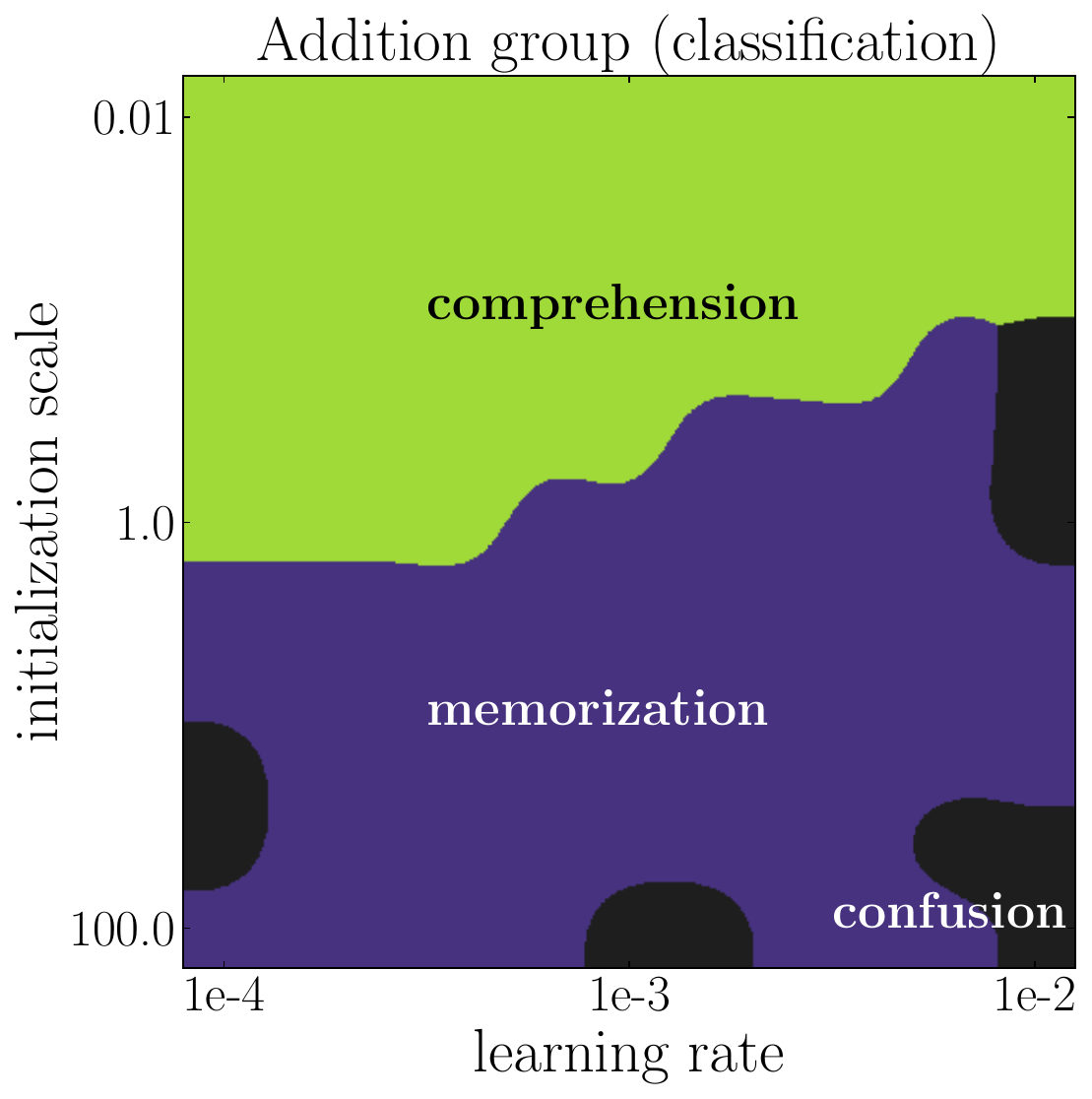}
    \caption{Phase diagrams of  decoder learning rate (x axis) and initialization (y axis) for the addition group (left: regression; right: classification). Small intialization scale (top) leads to comprehension.}
    \label{fig:grokking_pd_init}
\end{figure}

{\bf Representation weight decay} controls the magnitude of embedding vectors. Shown in Figure~\ref{fig:grokking_pd_reprwd}, we see the representation weight decay in general does not affect model performance much.
\begin{figure}
    \centering
    \includegraphics[width=0.45\linewidth]{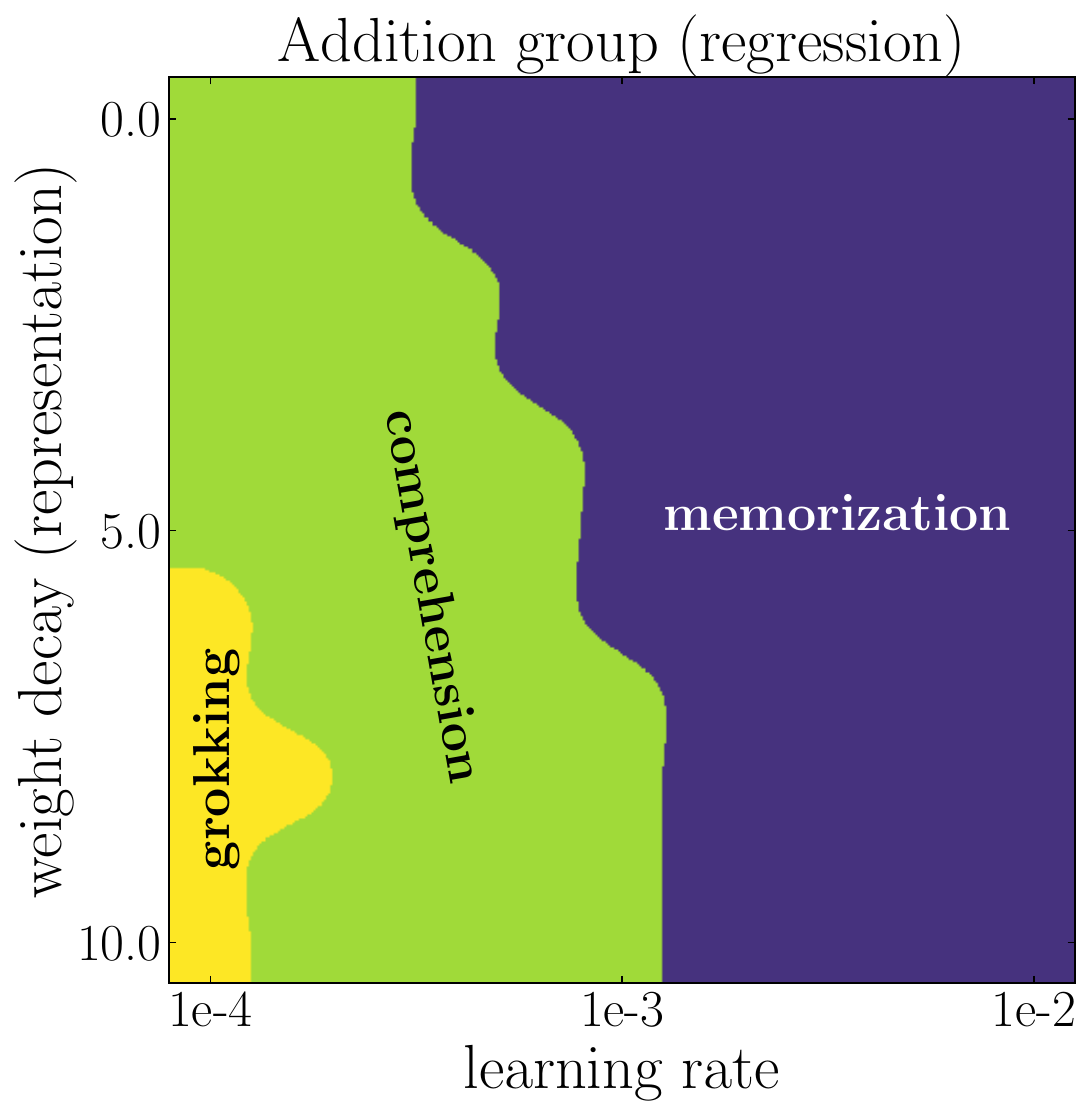}
    \includegraphics[width=0.45\linewidth]{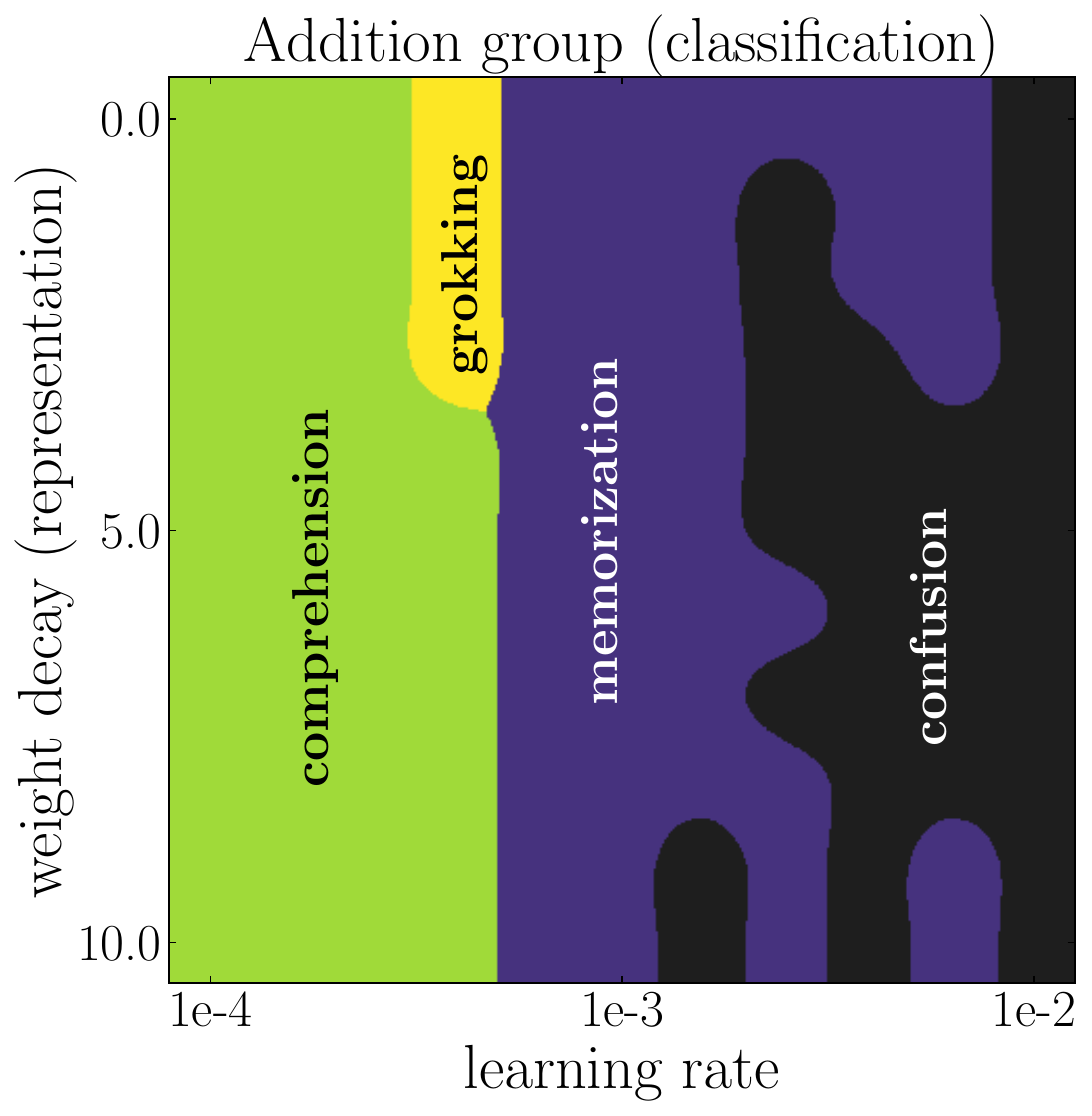}
    \caption{Phase diagrams of  decoder learning rate (x axis) and representation weight decay (y axis) for the addition group (left: regression; right: classification). Representation weight decay  does not affect model performance much.}
    \label{fig:grokking_pd_reprwd}
\end{figure}

\section{General groups}\label{app:non-abelian}

\subsection{Theory}
We focused on Abelian groups for the most part of the paper. This is, however, simply due to pedagogical reasons. In this section, we show that it is straight-forward to extend definitions of parallelograms and representation quality index (RQI) to general non-Abelian groups. We will also show that most (if not all) qualitative results for the addition group also apply to the permutation group.

{\bf Matrix representation for general groups} Let us first review the definition of group representation. A representation of a group $G$ on a vector space $V$ is a group homomorphism from $G$ to ${\rm GL}(V)$, the general linear group on $V$. That is, a representation is a map $\rho:G\to{\rm GL}(V)$ such that
\begin{equation}
    \rho(g_1g_2)=\rho(g_1)\rho(g_2),\quad \forall g_1,g_2\in G.
\end{equation}
In the case $V$ is of finite dimension $n$, it is common to identify ${\rm GL}(V)$ with $n$ by $n$ invertible matrices. The punchline is that: each group element can be represented as a matrix, and the binary operation is represented as matrix multiplication.

{\bf A new architecture for general groups} Inspired by the matrix representation, we embed each group element $a$ as a learnable matrix $\mat{E}_a\in\mathbb{R}^{d\times d}$ (as opposed to a vector), and manually do matrix multiplication before sending the product to the deocder for regression or classification. More concretly, for $a\circ b=c$, our architecture takes as input two embedding matrices $\mat{E}_a$ and $\mat{E}_b$ and aims to predict $\mat{Y}_c$ such that $\mat{Y}_c={\rm Dec}(\mat{E}_a\mat{E}_b)$, where $\mat{E}_a\mat{E}_b$ means the matrix multiplication of $\mat{E}_a$ and $\mat{E}_b$. The goal of this simplication is to disentangle learning the representation and learning the arithmetic operation (i.e, the matrix multiplication). We will show that, even with this simplification, we are still able to reproduce the characteristic grokking behavior and other rich phenomenon.

{\bf Generalized parallelograms} we define generalized parallelograms:  $(a,b,c,d)$ is a generalized parallelogram in the representation if $||\mat{E}_a\mat{E}_b-\mat{E}_c\mat{E}_d||_F^2\leq\delta$, where $\delta>0$ is a threshold to tolerate numerical errors. Before presenting the numerical results for the permutation group, we show an intuitive picture about how new parallelograms can be deduced from old ones for general groups, which is the key to generalization.

\begin{figure}
    \centering
    \includegraphics[width=1\linewidth]{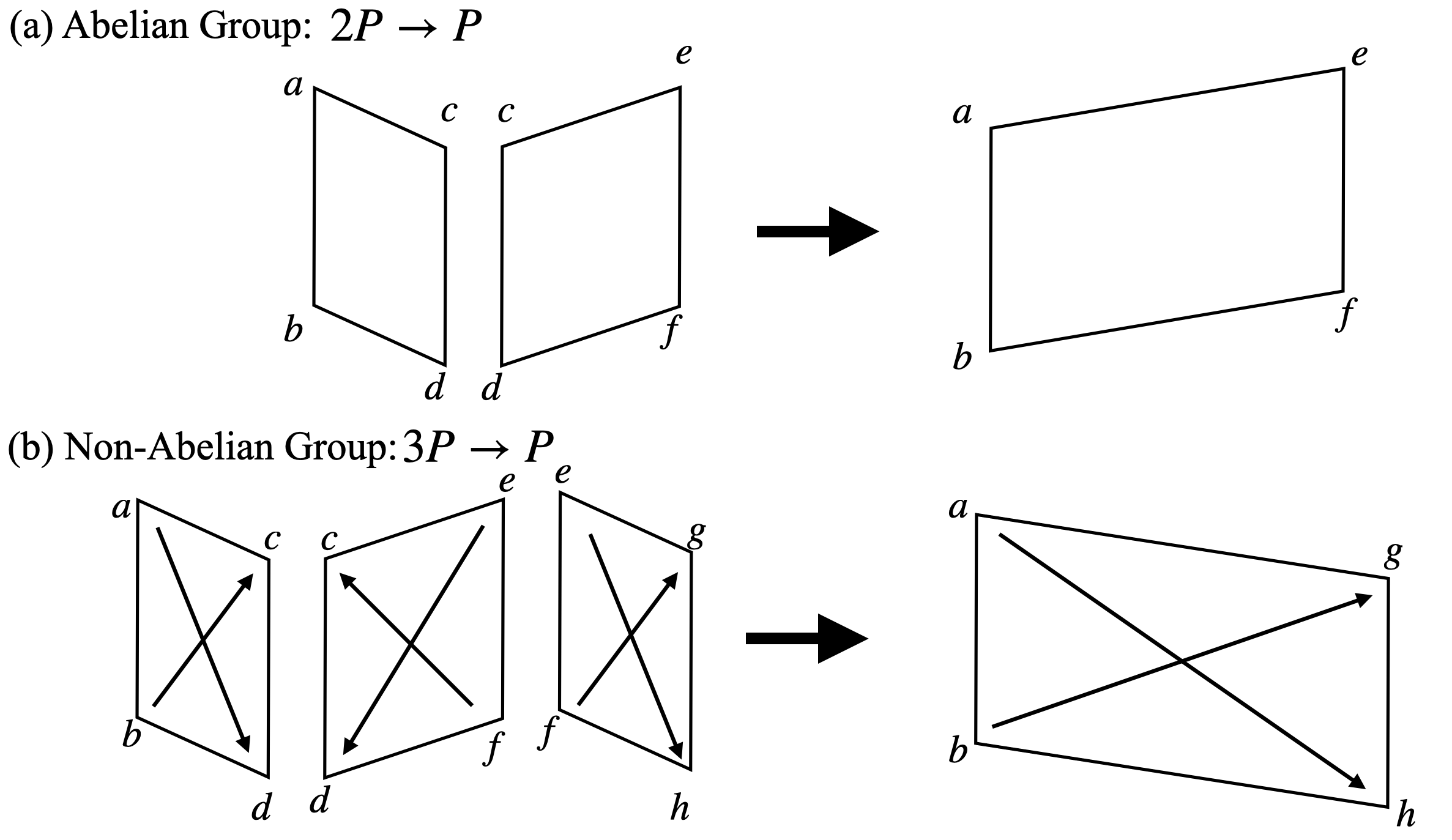}
    \caption{Deduction of parallelograms}
    \label{fig:P_deduction}
\end{figure}

{\bf Deduction of parallelograms} We first recall the case of the Abelian group (e.g., addition group). As shown in Figure~\ref{fig:P_deduction}, when $(a,d,b,c)$ and $(c,f,d,e)$ are two parallelograms, we have
\begin{equation}
    \begin{aligned}
    \mat{E}_a + \mat{E}_d = \mat{E}_b + \mat{E}_c, \\
    \mat{E}_c + \mat{E}_f = \mat{E}_d + \mat{E}_d. \\
    \end{aligned}
\end{equation}
We can derive that $\mat{E}_a + \mat{E}_f = \mat{E}_b + \mat{E}_e$ implying that $(a,f,b,e)$ is also a parallelogram. That is, for Abelian groups, two parallelograms are needed to deduce a new parallelogram.

For the non-Abelian group, if we have only two parallelograms such that
\begin{equation}
    \begin{aligned}
    \mat{E}_a\mat{E}_d = \mat{E}_b\mat{E}_c, \\
    \mat{E}_f\mat{E}_c = \mat{E}_e\mat{E}_d,
    \end{aligned}
\end{equation}
we have $\mat{E}_b^{-1}\mat{E}_a=\mat{E}_c\mat{E}_d^{-1}=\mat{E}_f^{-1}\mat{E}_e$, but this does not lead to something like $\mat{E}_f\mat{E}_a=\mat{E}_e\mat{E}_b$, hence useless for generalization. However, if we have a third parallelogram such that
\begin{equation}
    \mat{E}_e\mat{E}_h = \mat{E}_f\mat{E}_g
\end{equation}
we have $\mat{E}_b^{-1}\mat{E}_a=\mat{E}_c\mat{E}_d^{-1}=\mat{E}_f^{-1}\mat{E}_e=\mat{E}_g\mat{E}_h^{-1}$, equivalent to $\mat{E}_a\mat{E}_h=\mat{E}_b\mat{E}_g$, thus establishing a new parallelogram $(a,h,b,g)$. That is, for non-Abelian groups, three parallelograms are needed to deduce a new parallelogram.

\subsection{Numerical Results}

In this section, we conduct numerical experiments on a simple non-abelian group: the permutation group $S_3$. The group has 6 group elements, hence the full dataset contains 36 samples. We embed each group element $a$ into a learnable $3\times 3$ embedding matrix $\mat{E}_a$. We adopt the new architecture described in the above subsection: we hard code matrix multiplication of two input embedding matrices before feeding to the decoder. After defining the generalized parallelogram in the last subsection, we can continue to define RQI (as in Section \ref{sec:represenation}) and predict accuracy $\widehat{\rm Acc}$ from representation
(as in appendix \ref{app:acc}). We also compute the number of steps needed to reach ${\rm RQI}=0.95$.

{\bf Representation} We flatten each embedding matrix into a vector, and apply principal component analysis (PCA) to the vectors. We show the first three principal components of these group elements in Figure~\ref{fig:permutation_repr}. On the plane of ${\rm PC}1$ and ${\rm PC}3$, the six points are organized as a hexagon. 
\begin{figure}[ht]
    \centering
    \includegraphics[width=1\linewidth]{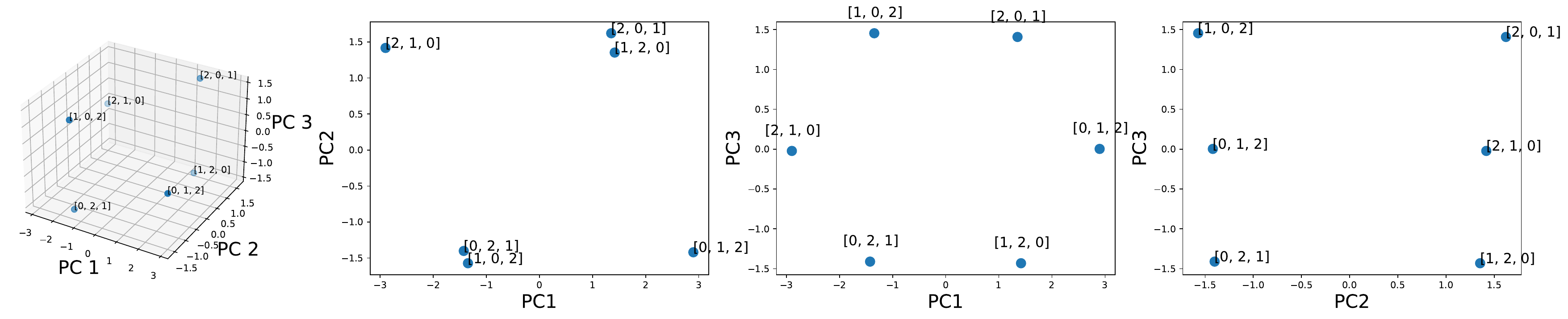}
    \caption{Permuation group $S_3$. First three principal components of six embedding matrices $\mathbb{R}^{3\times 3}$.}
    \label{fig:permutation_repr}
\end{figure}

\begin{figure}[ht]
    \centering
    \includegraphics[width=1\linewidth]{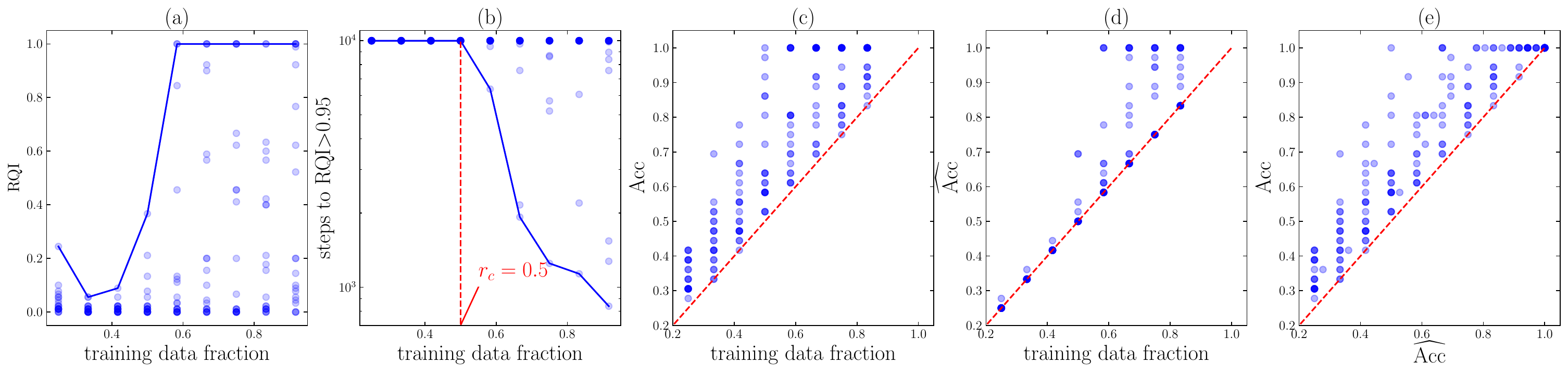}
    \caption{Permutation group $S_3$. (a) RQI increases as training set becomes larger. Each scatter point is a random seed, and the blue line is the highest RQI obtained with a fixed training set ratio; (b) steps to reach ${\rm RQI}>0.95$. The blue line is the smallest number of steps required. There is a phase transition around $r_c=0.5$. (c) real accuracy ${\rm Acc}$; (d) predicted accuracy $\widehat{\rm Acc}$; (e) comparison of ${\rm Acc}$ and $\widehat{\rm Acc}$: $\widehat{\rm Acc}$ serves as a lower bound of ${\rm Acc}$. }
    \label{fig:permutation_results}
\end{figure}

{\bf RQI} In Figure~\ref{fig:permutation_results} (a), we show RQI as a function of training data fraction. For each training data fraction, we run 11 random seeds (shown as scatter points), and the blue line corresponds to the highest RQI.

{\bf Steps to reach RQI$=0.95$} In Figure~\ref{fig:permutation_results} (b), we whow the steps to reach ${\rm RQI}>0.95$ as a function of training data fraction, and find a phase transition at $r=r_c=0.5$. The blue line corresponds to the best model (smallest number of steps).

{\bf Accuracy} The real accuracy ${\rm Acc}$ is shown in Figure~\ref{fig:permutation_results} (c), while the predicted accuracy $\widehat{\rm Acc}$ (calculated from RQI) is shown in Figure~\ref{fig:permutation_results} (d). Their comparison is shown in (e): $\widehat{\rm Acc}$ is a lower bound of ${\rm Acc}$, implying that there must be some generalization mechanism beyond RQI.

{\bf Phase diagram} We investigate how the model performance varies under the change of two knobs: decoder learning rate and decoder weight decay. We calculate the number of steps to training accuracy $\geq 0.9$ and validation accuracy $\geq 0.9$, respectively, shown in Figure~\ref{fig:grokking_pd} (d).

\section{Effective theory for image classification}\label{app:mnist_effective}

\begin{figure}[htbp]
    \centering
    \includegraphics[width=1.0\linewidth]{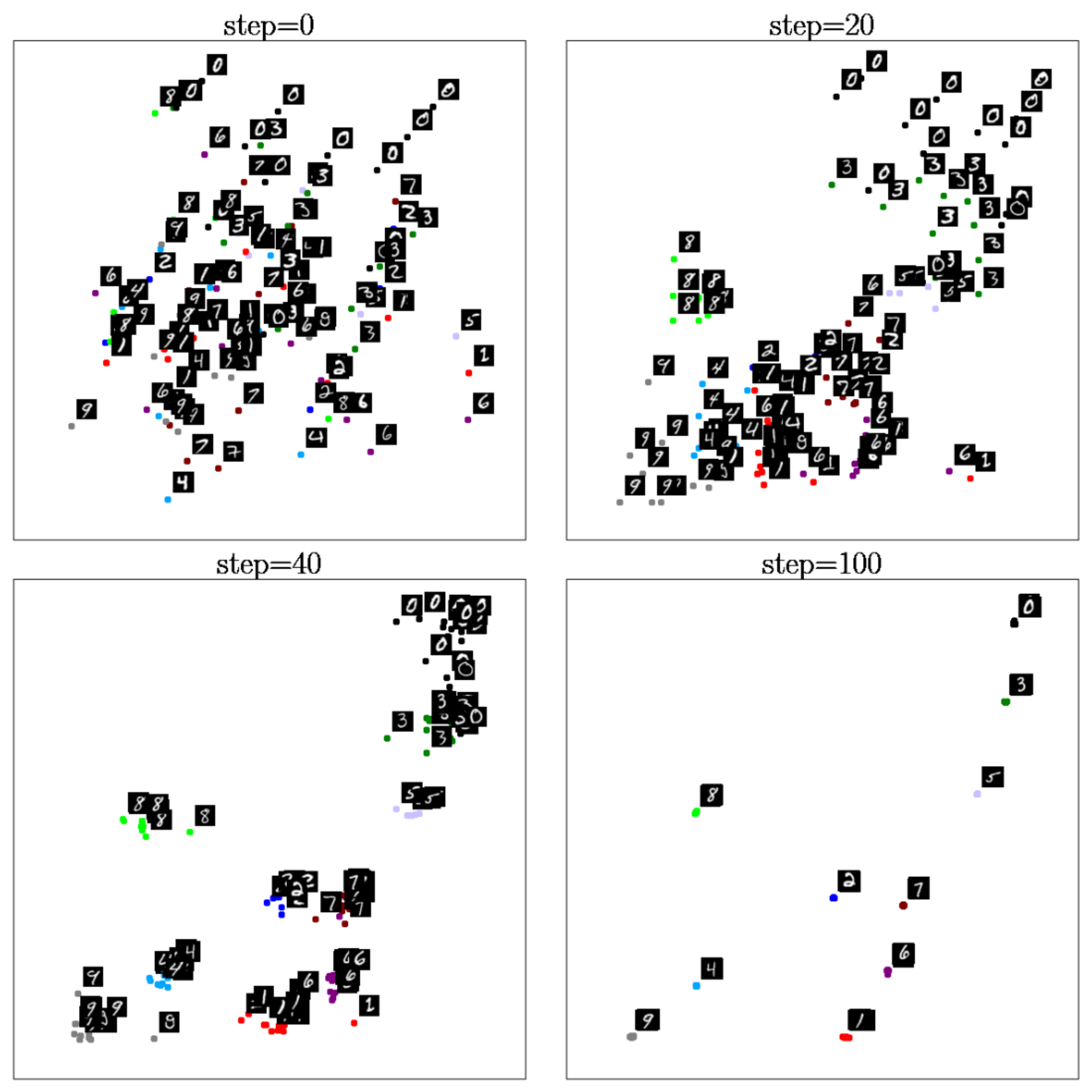}
    \caption{Our effective theory applies to MNIST image classifications. Same-class images collapse to their class-means, while  class-means of different classes stay separable. As such, the effective theory serves as a novel self-supervised learning method, as well as shed some light on neural collapse. Please see texts in Appendix \ref{app:mnist_effective}.}
    \label{fig:mnist_effective}
\end{figure}


In this section, we show our effective theory proposed in Section \ref{sec:effective_theory} can generalize beyond algorithmic datasets. In particular, we will apply the effective theory to image classifications. We find that: (i) The effective theory naturally gives rise to a novel self-supervised learning method, which can provably avoid mode collapse without contrastive pairs. (ii) The effective theory can shed light on the neural collapse phenomenon~\cite{papyan2020prevalence}, in which same-class representations collapse to their class-means.

We first describe how the effective theory applies to image classification. The basic idea is again that, similar to algorithmic datasets, neural networks try to develop a structured representation of the inputs based on the relational information between samples (class labels in the case of image classification, sum parallelograms in the case of addition, etc.).
The effective theory has two ingredients: (i) samples with the same label are encouraged to have similar representations; (ii) the effective loss function is scale-invariant to avoid all representations collapsing to zero (global collapse). As a result, an effective loss for image classification has the form

\begin{equation}\label{eq:effective_MNIST}
    \ell_{\rm eff} = \frac{\ell}{Z},\quad  \ell = \sum_{(\mat{x},\mat{y})\in P} |\mat{f}(\mat{x})-\mat{f}(\mat{y})|^2, \quad Z=\sum_{\mat{x}} |\mat{f}(\mat{x})|^2 
\end{equation}
where $\mat{x}$ is an image, $\mat{f}(\mat{x})$ is its representation, $(\mat{x},\mat{y})\in P$ refers to unique pairs $\mat{x}$ and $\mat{y}$ that have the same label. Scale invariance means the loss function $\ell_{\rm eff}$ does not change under the linear scaling $\mat{f}(\mat{x})\to a\mat{f}(\mat{x})$.

{\bf Relation to neural collapse} It was observed in ~\cite{papyan2020prevalence} that image representations in the penultimate layer of the model have some interesting features: (i) representations of same-class images collapse to their class-means; (ii) class-means of different classes develop into an equiangular tight frame. Our effective theory is able to predict the same-class collapse, but does not necessarily put class-means into equiangular tight frames. We conjecture that little explicit repulsion among different classes can help class-means develop into an equiangular tight frame, similar to electrons developing into lattice structures on a sphere under repulsive Coulomb forces (the Thomson problem ~\cite{enwiki:1091431454}). We would like to investigate this modification of the effective theory in the future.

{\bf Experiment on MNIST} We directly apply the effective loss Eq.~(\ref{eq:effective_MNIST}) to the MNIST dataset. Firstly, each image $\mat{x}$ is randomly encoded to a 2D embedding $\mat{f}(\mat{x})$ via the same encoder MLP whose weights are randomly initialized. We then train these embeddings by minimizing the effective loss $\ell_{\rm eff}$ with an Adam optimizer ($10^{-3}$ learning rate) for 100 steps. We show the evolution of these embeddings in Figure \ref{fig:mnist_effective}. Images of the same class collapse to their class-means, and different class-means do not collapse. This means that our effective theory can give rise to a good representation learning method which only exploits non-contrastive relational information in datasets.

{\bf Link to self-supervised learning} Note that $\ell$ itself is vulnerable to global collapse, in the context of Siamese learning without contrastive pairs. Various tricks (e.g., decoder with momentum, stop gradient) ~\cite{grill2020bootstrap, chen2021exploring} have been proposed to avoid global collapse. However, the reasons why these tricks can avoid global collapse are unclear. We argue $\ell$ fails simply because $\ell\to a^2\ell$ under scaling $\mat{f}(\mat{x})\to a\mat{f}(\mat{x})$ so gradient descent on $\ell$ encourage $a\to 0$. Based on this picture, our effective theory provides another possible fix: make the loss function $\ell$ scale-invariant (by the normalized loss $\ell_{\rm eff}$), so the gradient flow has no incentive to change representation scales. In fact, we can prove that the gradient flow on $\ell_{\rm eff}$ preserve $Z$ (variance of representations) so that global collapse is avoided provably:

\begin{equation}
\begin{aligned}
    \frac{\partial\ell_{\rm eff}}{\partial\mat{f}(\mat{x})} & = \frac{1}{Z}\frac{\partial\ell}{\partial\mat{f}(\mat{x})} - \frac{l}{Z^2}\frac{\partial Z}{\partial \mat{f}(\mat{x})}=\frac{2}{Z}\sum_{\mat{y}\sim\mat{x}}(\mat{f}(\mat{x})-\mat{f}(\mat{y}))-\frac{2\ell}{Z^2}\mat{f}(\mat{x}), \\
    \frac{dZ}{dt} & = 2 \sum_{\mat{x}}\mat{f}(\mat{x})\cdot \frac{d\mat{f}(\mat{x})}{dt} =2 \sum_{\mat{x}}\mat{f}(\mat{x})\cdot\frac{\partial\ell_{\rm eff}}{\partial\mat{f}(\mat{x})} \\
    & =\frac{4}{Z}\sum_{\mat{x}}\mat{f}(\mat{x})\cdot(\sum_{\mat{y}\sim\mat{x}}(\mat{f}(\mat{x})-\mat{f}(\mat{y}))-\frac{\ell}{Z}\mat{f}(\mat{x})) \\
    & = \frac{4}{Z}\big[\sum_{\mat{x}}\mat{f}(\mat{x})\cdot\sum_{\mat{y}\sim\mat{x}}(\mat{f}(\mat{x})-\mat{f}(\mat{y}))-\sum_{\mat{x}}\frac{\ell}{Z}|\mat{f}(\mat{x})|^2 \big]\\
    & = 0.
\end{aligned}
\end{equation}
where we use the fact that 
\begin{equation}
    \sum_{\mat{x}}\mat{f}(\mat{x})\cdot\sum_{\mat{y}\sim\mat{x}}(\mat{f}(\mat{x})-\mat{f}(\mat{y}))=\sum_{(\mat{x},\mat{y}) \in P} (\mat{f}(\mat{x})-\mat{f}(\mat{y}))\cdot (\mat{f}(\mat{x})-\mat{f}(\mat{y})) = \ell
\end{equation}

\section{Grokking on MNIST}
\label{app:mnist_grok}

To induce grokking on MNIST, we make two nonstandard decisions: (1) we reduce the size of the training set from 50k to 1k samples (by taking a random subset) and (2) we increase the scale of the weight initialization distribution (by multiplying the initial weights, sampled with Kaiming uniform initialization, by a constant $> 1$). 

The choice of large initializations is justified by ~\citep{xu2019training, zhang2020type,liu2022omnigrok} which find large initializations overfit data easily but prone to poor generalization. Relevant to this, initialization scale is found to regulate ``kernel'' vs ``rich'' learning regimes in networks~\cite{pmlr-v125-woodworth20a}.


With these modifications to training set size and initialization scale, we train a depth-3 width-200 MLP with ReLU activations with the AdamW optimizer. We use MSE loss with one-hot targets, rather than cross-entropy. With this setup, we find that the network quickly fits the train set, and then much later in training validation accuracy improves, as shown in Figure~\ref{fig:mnist-learning-curve}. This closely follows the stereotypical grokking learning, first observed in algorithmic datasets. 

With this setup, we also compute a phase diagram over the model weight decay and the last layer learning rate. See Figure~\ref{fig:mnist-phase}. While in MLPs it is less clear what parts of the network to consider the ``encoder'' vs the ``decoder'', for our purposes here we consider the last layer to be the ``decoder'' and vary its learning rate relative to the rest of the network. The resulting phase diagram has some similarity to Figure~\ref{fig:degrok}. We observe a ``confusion''phase in the bottom right (high learning rate and high weight decay), a ``comprehension'' phase bordering it, a ``grokking'' phase as one decreases weight decay and decoder learning rate, and a ``memorization`` phase at low weight decay and low learning rate. Instead of an accuracy threshold of 95\%, we use a threshold of 60\% here for validation accuracy for runs to count as comprehension or grokking. This phase diagram demonstrates that with sufficient regularization, we can again ``de-grok'' learning.



We also investigate the effect of training set size on time to generalization on MNIST. We find a result similar to what Power et al.~\cite{power2022grokking} observed, namely that generalization time increases rapidly once one drops below a certain amount of training data. See Figure~\ref{fig:generalization-vs-data-mnist}.

\begin{figure}[ht]
    \centering
    \includegraphics[width=3in]{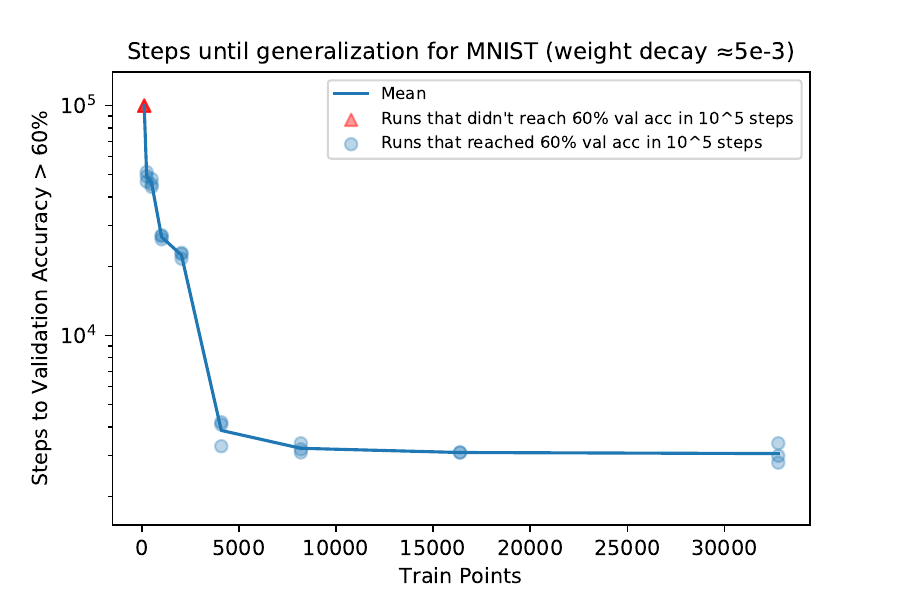}
    \caption{Time to generalize as a function of training set size, on MNIST.}
    \label{fig:generalization-vs-data-mnist}
\end{figure}


\section{Lottery Ticket Hypothesis Connection}
\label{app:lth}
\begin{figure}[ht]
  \centering
  \includegraphics[width=1\linewidth]{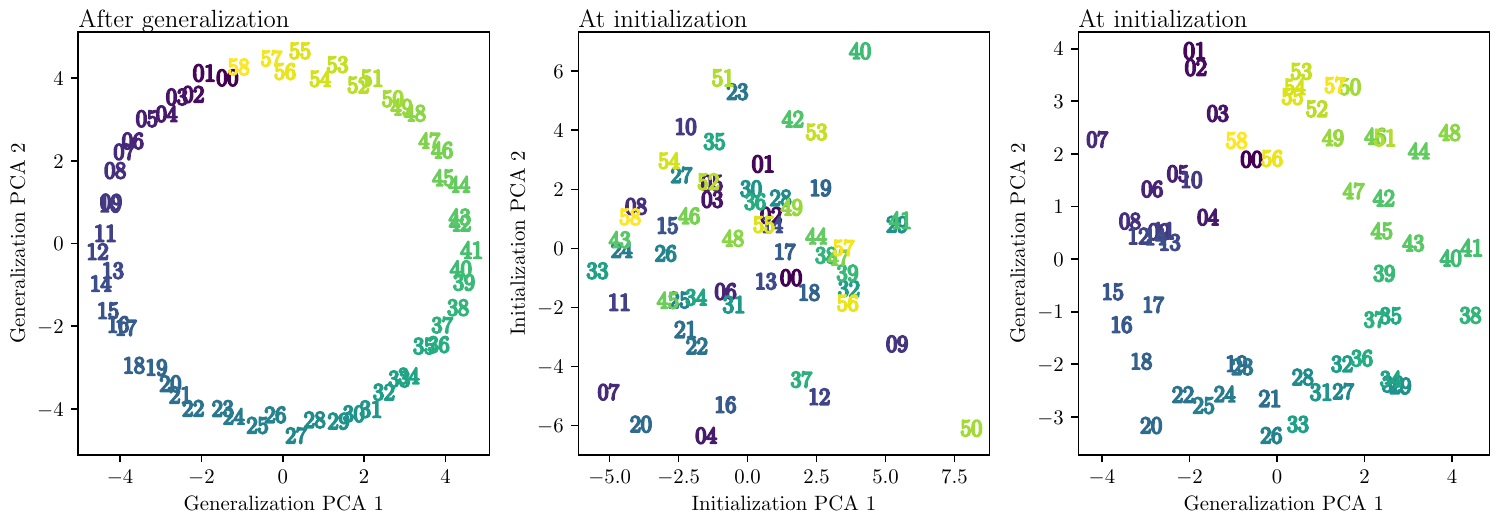}
    \caption{\textbf{(Left)} Input embeddings after generalization projected on their first 2 principal components.\textbf{(Center)} Input embeddings at initialization projected on their first 2 principal components. \textbf{(Right)} Input embeddings at initialization projected on the first 2 principal components of the embeddings after generalization at the end of training (same PCA as the left figure).}  
    \label{fig:pca_lth}
\end{figure}

In Figure \ref{fig:pca_lth}, we show the projection of the learned embeddings after generalization to their first two principal components. Compared to the projection at initialization, structure clearly emerges in embedding space when the neural network is able to generalize ($>99\%$ validation accuracy). What is intriguing is that the projection of the embeddings at initialization to the principal components of the embeddings at generalization seem to already contain much of that structure. In this sense, the structured representation necessary for generalization already existed (partially) at initialization. The training procedure essentially prunes other unnecessary dimensions and forms the required parallelograms for generalization. This is a nonstandard interpretation of the lottery ticket hypothesis where the winning tickets are not weights or subnetworks but instead particular axes or linear combinations of the weights (the learned embeddings).

\begin{figure}[!ht]
    \begin{subfigure}{1\textwidth}
        \includegraphics[width=1\linewidth]{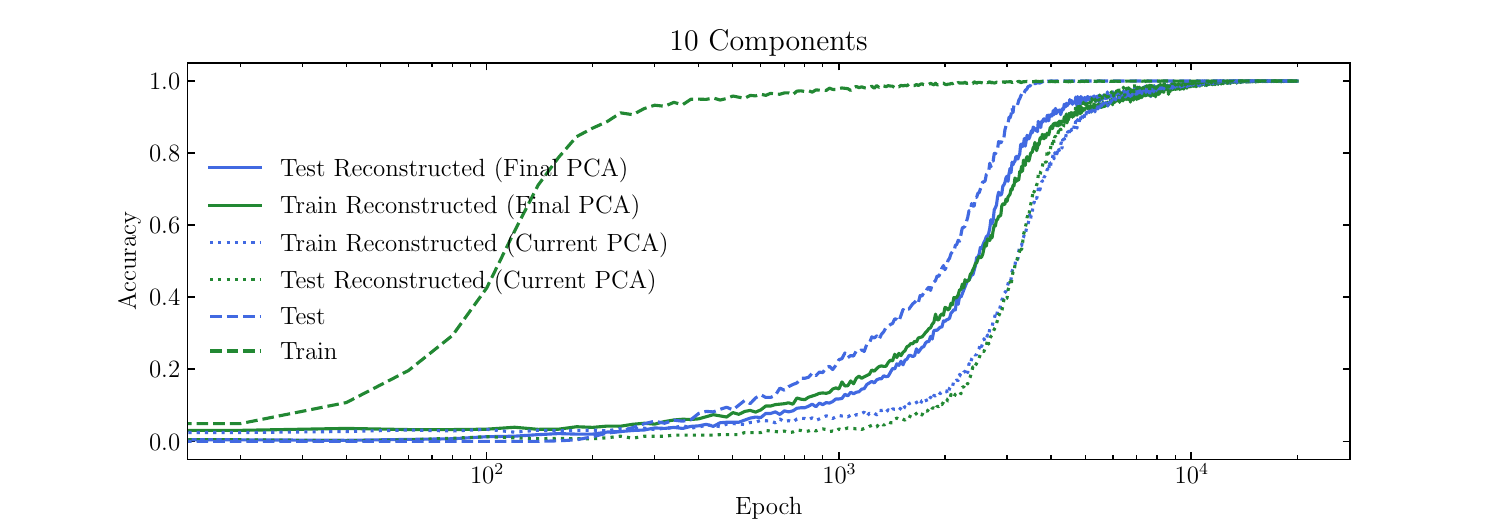}
    \end{subfigure}
    \begin{subfigure}{1\textwidth}
        \includegraphics[width=1\linewidth]{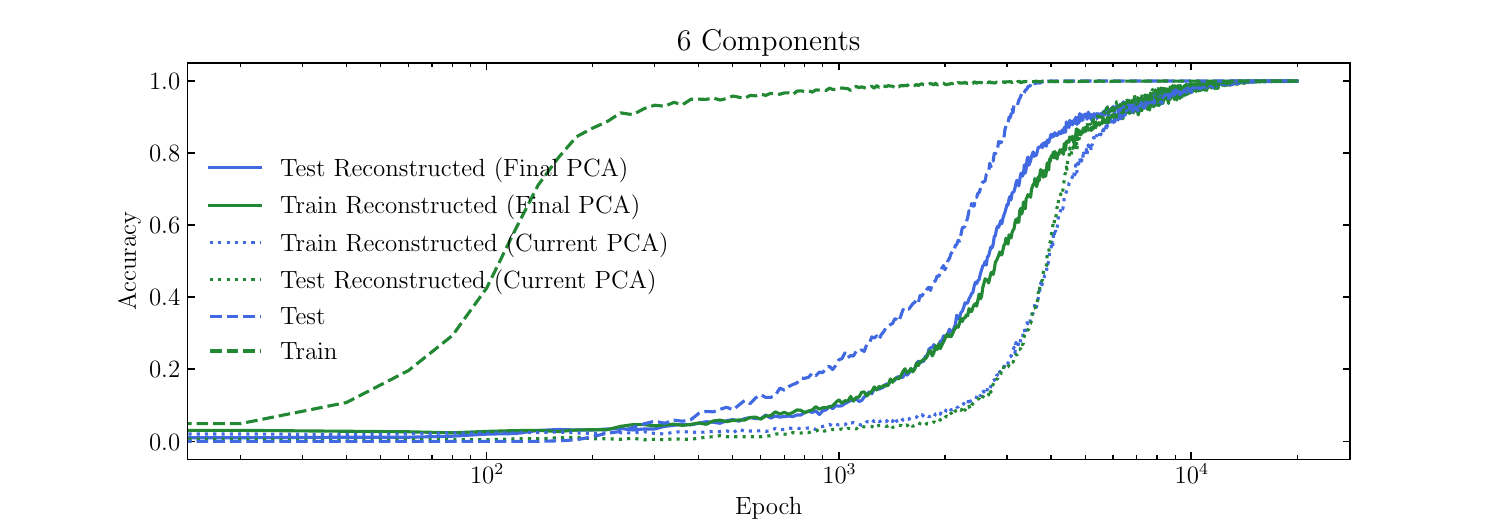}
    \end{subfigure}
    \begin{subfigure}{1\textwidth}
        \includegraphics[width=1\linewidth]{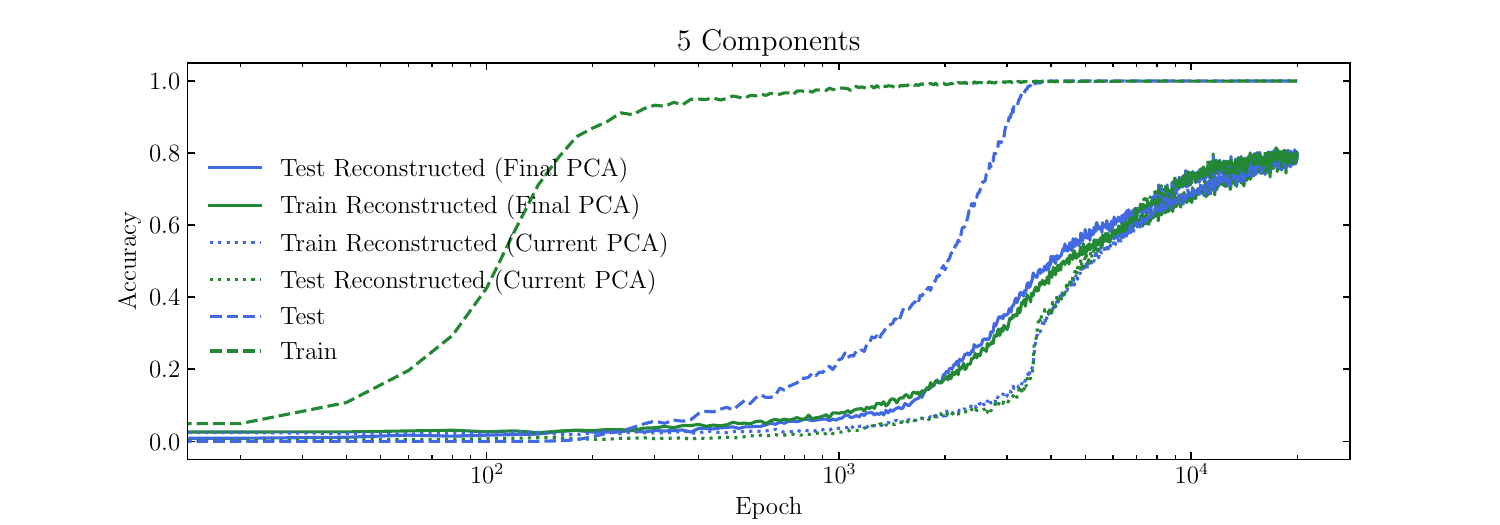}
    \begin{subfigure}{1\textwidth}
        \includegraphics[width=1\linewidth]{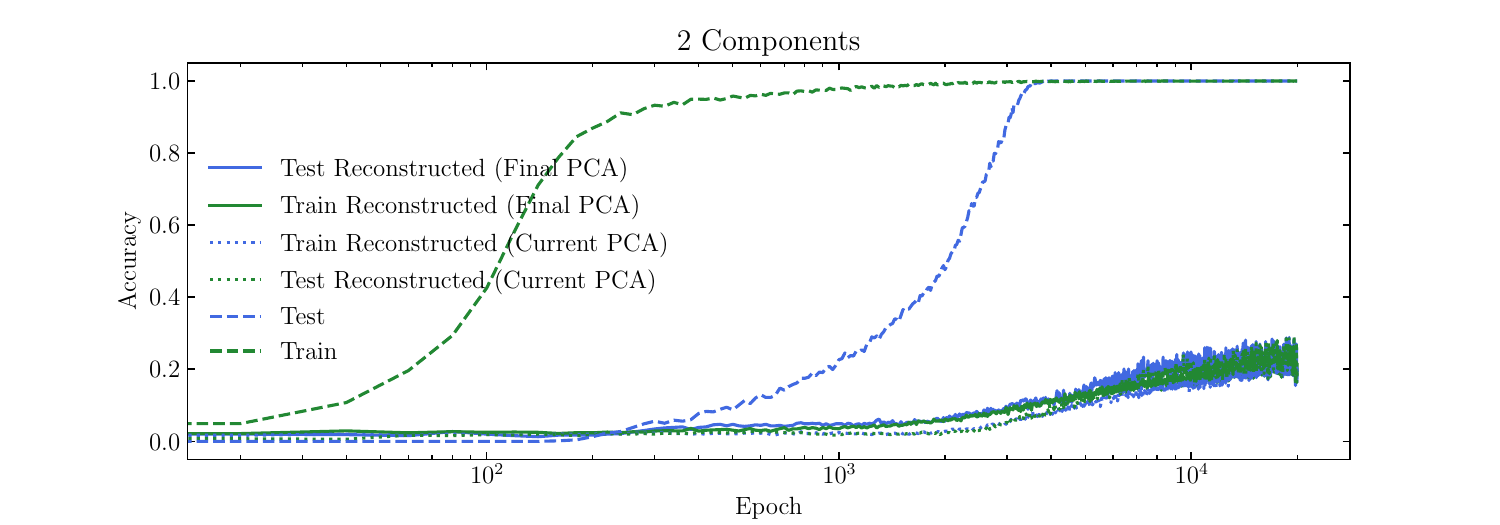}
    \end{subfigure}
    \end{subfigure}
  \centering
     \caption{Train and test accuracy computed while using actual embeddings (dashed line) and embeddings projected onto and reconstructed from their first $n$ principal components (dotted lines) and, finally, using embeddings projected onto and reconstructed from the first $n$ PCs of the embeddings at the end of training (solid lines).} 
    \label{fig:acc_lth}
\end{figure}

In Figure \ref{fig:acc_lth}, we show the original training curves (dashed lines). In solid lines, we recompute accuracy with models which use embeddings that are projected onto the $n$ principal components of the embeddings at the end of training (and back). Clearly, the first few principal components contain enough information to reach $99\%$ accuracy. The first few PCs explain the most variance by definition, however, we note that this is not necessarily the main reason for why they can generalize so well. In fact, embeddings reconstructed from the PCA at the end of training (solid lines) perform better than current highest variance axes (dotted line).
This behavior is consistent across seeds.

\newpage
\section{Derivation of the effective loss}
In this section, we will further motivate the use of our effective loss to study the dynamics of representation learning by deriving it from the gradient flow dynamics on the actual MSE loss in linear regression. The loss landscape of a neural network is in general nonlinear, but the linear case may shed some light on how the effective loss can be derived from actual loss. For a sample $\mat{r}$ (which is the sum of two embeddings $\mat{E}_i$ and $\mat{E}_j$), the prediction of the linear network is $D(\mat{r})= \mat{A}\mat{r}+\mat{b}$. The loss function is ($\mat{y}$ is its corresponding label):
\begin{equation}\label{eq:loss}
    \ell = \underbrace{\frac{1}{2}|\mat{A}\mat{r}+\mat{b}-\mat{y}|^2}_{\ell_{\rm pred}}+\underbrace{\frac{\gamma}{2}||\mat{A}||_F^2}_{\ell_{\rm reg}},
\end{equation}
where the first and the second term are prediction error and regularization, respectively. Both the model $(\mat{A},\mat{b})$ and the input $\mat{r}$ are updated via gradient flow, with learning rate $\eta_A$ and $\eta_x$, respectively:
\begin{equation}
    \frac{d\mat{A}}{dt} = -\eta_A\frac{\partial \ell}{\partial \mat{A}}, \frac{d\mat{b}}{dt} = -\eta_A\frac{\partial \ell}{\partial \mat{b}}, \frac{d\mat{r}}{dt} = -\eta_x\frac{\partial \ell}{\partial \mat{r}}.
\end{equation}
Inserting $\ell$ into the above equations, we obtain the gradient flow:
\begin{equation}
    \begin{aligned}
    &\frac{d\mat{A}}{dt} = - \eta_A\frac{\partial\ell}{\partial\mat{A}} = -\eta_A [\mat{A}(\mat{r}\mat{r}^T+\gamma)+(\mat{b}-\mat{y})\mat{r}^T], \\
    &\frac{d\mat{b}}{dt} = -\eta_A\frac{\partial\ell}{\partial\mat{b}} =  -\eta_A(\mat{A}\mat{r}+\mat{b}-\mat{y}) \\
    &\frac{d\mat{r}}{dt}=-\eta_x\frac{\partial\ell}{\partial\mat{r}}=-\eta_x \mat{A}^T(\mat{A}\mat{r}+\mat{b}-\mat{y}).
    \end{aligned}
\end{equation}
For the $d\mat{b}/dt$ equation, after ignoring the $\mat{A}\mat{r}$ term and set the initial condition $\mat{b}(0)=\mat{0}$, we obtain analytically $\mat{b}(t) = (1-e^{-2\eta_At})\mat{y}$. Inserting this into the first and third equations, we have
\begin{equation}
\begin{aligned}
    &\frac{d\mat{A}}{dt} = -\eta_A[\mat{A}(\mat{r}\mat{r}^T+\gamma)-e^{-2\eta_At}\mat{y}\mat{r}^T], \\
    &\frac{d\mat{r}}{dt} = \underbrace{- \eta_x\mat{A}^T\mat{A}\mat{r}}_{\rm internal\  interaction} + \underbrace{ \eta_xe^{-2\eta_At}\mat{A}^T\mat{y}}_{\rm external\  force}.
\end{aligned}
\end{equation}
For the second equation on the evolution of $d\mat{r}/dt$, we can artificially decompose the right hand side into two terms, based on whether they depend on the label $\mat{y}$. In this way, we call the first term "internal interaction" since it does not depend on $\mat{y}$, while the second term "external force". Note this distinction seems a bit artificial from a mathematical perspective, but it can be conceptually helpful from a physics perspective. We will show below the internal interaction term is important for representations to form.
Because we are interested in how two samples interact, we now consider another sample at $\mat{r}'$, and the evolution becomes
\begin{equation}
\begin{aligned}
    &\frac{d\mat{A}}{dt} = -\eta_A[\mat{A}(\mat{r}\mat{r}^T+\mat{r}'\mat{r}'^T+2\gamma)-e^{-2\eta_At}\mat{y}(\mat{r}+\mat{r}')^T], \\
    &\frac{d\mat{r}}{dt} = - \eta_x\mat{A}^T\mat{A}\mat{r} + \eta_xe^{-2\eta_At}\mat{A}^T\mat{y}, \\
    & \frac{d\mat{r}'}{dt} = - \eta_x\mat{A}^T\mat{A}\mat{r}' + \eta_xe^{-2\eta_At}\mat{A}^T\mat{y}. \\
\end{aligned}
\end{equation}
Subtracting $d\mat{r}/dt$ by $d\mat{r}'/dt$ and setting $\mat{r}'=-\mat{r}$, the above equations further simply to
\begin{equation}\label{eq:Ar}
    \begin{aligned}
    &\frac{d\mat{A}}{dt} = -2\eta_A\mat{A}(\mat{r}\mat{r}^T+\gamma), \\
    &\frac{d\mat{r}}{dt} = - \eta_x\mat{A}^T\mat{A}\mat{r}. \\
    \end{aligned}
\end{equation}

The second equation implies that the pair of samples interact via a quadratic potential $U(\mat{r})=\frac{1}{2}\mat{r}^T\mat{A}^T\mat{A}\mat{r}$, leading to a linear attractive force $f(r)\propto r$. We now consider the adiabatic limit where $\eta_A\to 0$.

{\bf The adiabatic limit} Using the standard initialization (e.g., Xavier initialization) of neural networks, we have $\mat{A}_0^T\mat{A}_0\approx\mat{I}$. As a result, the quadratic potential becomes $U(\mat{r})=\frac{1}{2}\mat{r}^T\mat{r}$, which is time-independent because $\eta_A\to 0$. We are now in the position to analyze the addition problem. For two samples $\mat{x}^{(1)}=\mat{E}_i+\mat{E}_j$ and $\mat{x}^{(2)}=\mat{E}_m+\mat{E}_n$ with the same label ($i+j=m+n$), they contribute to an interaction term 
\begin{equation}
    U(i,j,m,n) = \frac{1}{2} |\mat{E}_i+\mat{E}_j-\mat{E}_m-\mat{E}_n|_2^2.
\end{equation}
Averaging over all possible quadruples in the training dataset $D$, the total energy of the system is
\begin{equation}\label{eq:unnorm_H}
    \ell_0 = \sum_{(i,j,m,n)\in P_0(D)} \frac{1}{2}|\mat{E}_i+\mat{E}_j-\mat{E}_m-\mat{E}_n|_2^2/|P_0(D)|,
\end{equation}
where $P_0(D)=\{(i,j,m,n)|i+j=m+n,(i,j)\in D, (m,n)\in D\}$. To make it scale-invariant, we define the normalized Hamiltonian Eq.~(\ref{eq:unnorm_H}) as
\begin{equation}\label{eq:norm_H}
    \ell_{\rm eff} = \frac{\ell_0}{Z_0},\quad Z_0 = \sum_{i} |\mat{E}_i|_2^2
\end{equation}
which is the effective loss we used in Section \ref{sec:effective_theory}.

\end{document}